\documentclass[twoside,11pt]{article}

 \usepackage[preprint]{jmlr2e}

\usepackage{jmlr2e}
\usepackage{mathtools}
\usepackage{algorithm}
\usepackage{algorithmic}
\usepackage{xcolor}

\newtheorem{Assumption}{Assumption} 

\jmlrheading{, submitted}{}{}{July/20}{not yet}{meila00a}{Kaiyi Ji, Junjie Yang and Yingbin Liang}

\ShortHeadings{Theoretical Convergence of Multi-Step Model-Agnostic Meta-Learning}{Ji, Yang and Liang}
\firstpageno{1}

\begin{document}

\title{Theoretical Convergence of Multi-Step Model-Agnostic Meta-Learning}

\author{\name Kaiyi\ Ji \email ji.367@osu.edu \\
       \addr Department of Electrical and Computer Engineering\\
       The Ohio State University\\
      Columbus, OH 98195-4322, USA
       \AND
       \name Junjie\ Yang \email  yang.4972@osu.edu \\
       \addr Department of Electrical and Computer Engineering\\
       The Ohio State University\\
      Columbus, OH 98195-4322, USA
       \AND
       \name Yingbin\ Liang \email liang.889@osu.edu \\
       \addr Department of Electrical and Computer Engineering\\
       The Ohio State University\\
      Columbus, OH 98195-4322, USA}

\editor{}

\maketitle

\begin{abstract}
As a popular meta-learning approach, the model-agnostic meta-learning (MAML) algorithm has been widely used due to its  simplicity and effectiveness. However, the convergence of the general multi-step MAML still remains unexplored. In this paper, we develop a new theoretical framework to provide such convergence guarantee for two types of objective functions that are of interest in practice: (a) resampling case (e.g., reinforcement learning), where loss functions take the form in expectation and new data are sampled as the algorithm runs; and (b) finite-sum case (e.g., supervised learning), where loss functions take the finite-sum form with given samples. For both cases, we characterize the convergence rate and the computational complexity to attain an $\epsilon$-accurate solution for multi-step MAML in the general nonconvex setting. In particular, our results suggest that an inner-stage stepsize needs to be chosen inversely proportional to the number $N$ of inner-stage steps in order for $N$-step MAML to have guaranteed convergence.
From the technical perspective, we develop novel techniques to deal with the nested structure of the meta gradient for multi-step MAML, which can be of independent interest. 
\end{abstract}


\section{Introduction}
Meta-learning or learning to learn~\citep{thrun2012learning,naik1992meta,bengio1991learning} is a powerful tool  for  quickly learning new tasks by using the prior experience from  related tasks. Recent works have empowered this idea with neural networks, and their proposed meta-learning algorithms have been shown to enable fast learning over unseen tasks using only a few samples 
by efficiently extracting the knowledge from a range of observed tasks~\citep{santoro2016meta,vinyals2016matching,finn2017model}.  Current meta-learning algorithms can be generally categorized into metric-learning based~\citep{koch2015siamese,snell2017prototypical},  model-based~\citep{vinyals2016matching,munkhdalai2017meta}, and optimization-based~\citep{finn2017model,nichol2018reptile,rajeswaran2019meta} approaches. Among them, optimization-based meta-learning is  a  simple and effective approach used in a wide range of domains including classification/regression~\citep{rajeswaran2019meta}, reinforcement learning~\citep{finn2017model}, robotics~\citep{al2018continuous}, federated learning~\citep{chen2018federated}, and imitation learning~\citep{finn2017one}.

Model-agnostic meta-learning (MAML)~\citep{finn2017model} is a popular optimization-based method, which is simple and compatible generally  with  models  trained with gradient descents. MAML consists of two nested stages, where the inner stage runs a few steps of (stochastic) gradient descent for each individual task, and the outer stage updates the meta parameter over all the sampled tasks. The goal of MAML is to find a good meta initialization $w^*$ based on the observed tasks such that for a new task, starting from this $w^*$, a few (stochastic) gradient steps suffice to find a good model parameter.  Such an algorithm has been demonstrated to have superior empirical performance~\citep{antoniou2019train,grant2018recasting,zintgraf2018caml,nichol2018first}. Recently, the theoretical convergence of MAML has also been studied. Specifically, \cite{finn2019online} extended MAML to the online setting, and analyzed the regret for the strongly convex objective function. 
\cite{fallah2020convergence} provided an analysis for one-step MAML for general nonconvex functions, where each inner stage takes  a single stochastic gradient descent (SGD) step. 

In practice, the MAML training often takes {\em multiple} SGD steps at the inner stage, for example in \cite{finn2017model,antoniou2019train} for supervised learning and in~\cite{finn2017model,fallah2020provably} for reinforcement learning, in order to attain a higher test accuracy (i.e., better generalization performance) even at a price of higher computational cost. However, the theoretical convergence of such {\em multi-step} MAML algorithms has not been established yet. 
In fact, several mathematical challenges will arise in the theoretical analysis if the inner stage of MAML takes multiple steps. First, the meta gradient of multi-step MAML has a nested and recursive structure, which requires the performance analysis of an optimization path over a nested structure. In addition, multi-step update also yields a complicated bias error in the Hessian estimation as well as the statistical correlation between the Hessian and gradient estimators, both of which cause further difficulty in the analysis of the meta gradient. {\em The main contribution of this paper lies in the development of a new theoretical framework for analyzing the general {\em multi-step} MAML with techniques for handling the above challenges. }

\subsection{Main Contributions}

We develop a new theoretical framework, under which we characterize the convergence rate and the computational complexity to attain an $\epsilon$-accurate solution for {\em multi-step} MAML in the general nonconvex setting. Specifically, for the resampling case where each iteration needs sampling of fresh data (e.g., in reinforcement learning), our analysis enables to decouple the Hessian approximation error from the gradient approximation error based on a novel bound on the distance between two different inner optimization paths, which facilitates the analysis of the overall convergence of MAML. For the finite-sum case where the objective function is based on pre-assigned samples (e.g., supervised learning), we develop novel techniques to handle the difference between two losses over the training and  test sets in the  analysis. 

Our analysis provides a guideline for choosing the inner-stage stepsize at the order of $\mathcal{O}(1/N)$ and shows that  
$N$-step MAML is guaranteed to converge with the gradient and Hessian computation complexites  growing only linearly with $N$, which is consistent with the empirical observations in~\citealt{antoniou2019train}.
In addition, for problems where Hessians are small, e.g., most classification/regression meta-learning problems~\citep{finn2017model}, we show that the inner stepsize $\alpha$ can be set  larger while still maintaining the convergence,
which explains the empirical findings for MAML training in~\citealt{finn2017model,rajeswaran2019meta}.

\subsection{Related Work}

There are generally three types of meta-learning algorithms, which include optimization-based \citep{finn2017model,nichol2018reptile,rajeswaran2019meta}, metric-learning based \citep{koch2015siamese,snell2017prototypical}, and model-based \citep{vinyals2016matching,munkhdalai2017meta} approaches. Since the focus of this paper is on optimization-based meta-learning, we next discuss this type of meta-learning algorithms in more detail.

\vspace{0.2cm}

\noindent{\bf Optimization-based meta-learning.} Optimization-based meta-learning approaches have been widely used due to its simplicity and efficiency~\citep{li2017meta,ravi2016optimization,finn2017model}. 
 As a pioneer along this line, MAML~\citep{finn2017model} aims to find an initialization such that gradient descent from it achieves fast adaptation. Many follow-up studies~\citep{grant2018recasting,finn2019online,jerfel2018online,finn2018meta,finn2018probabilistic,mi2019meta,liu2019taming,rothfuss2019promp,foerster2018dice,fallah2020convergence,raghu2020rapid, collins2020distribution} have extended MAML from different perspectives. 
 For example, \cite{finn2019online} provided a follow-the-meta-leader extension of MAML for online learning.  
 Alternatively to meta-initialization algorithms such as MAML, meta-regularization approaches aim to learn a good bias for a regularized empirical risk  
minimization problem for intra-task learning~\citep{alquier2017regret, denevi2018learning,denevi2018incremental,denevi2019learning,rajeswaran2019meta,balcan2019provable,zhou2019efficient}. 
\cite{balcan2019provable} formalized a connection between meta-initialization and meta-regularization from an online learning perspective.  
\cite{zhou2019efficient} proposed an efficient meta-learning approach based on a  minibatch proximal update.~\cite{raghu2020rapid} proposed an efficient variant of MAML named ANIL (Almost No Inner Loop) by adapting only a small subset (e.g., head) of neural network parameters in the inner loop. 

Various Hessian-free MAML algorithms have been proposed to avoid  the costly computation of second-order derivatives, which include but not limited to FOMAML \citep{finn2017model}, Reptile \citep{nichol2018reptile}, ES-MAML \citep{song2020simple}, HF-MAML \citep{fallah2020convergence}. In particular, FOMAML \citep{finn2017model} omits all second-order derivatives in its meta-gradient computation, HF-MAML \citep{fallah2020convergence} estimates the meta gradient in one-step MAML using Hessian-vector product approximation.  This paper focuses on the first MAML algorithms, but the techniques here can be extended to analyze the Hessian-free multi-step MAML.

\vspace{0.2cm}

\noindent{\bf Optimization theory for meta-learning.} Theoretical property of MAML was initially established in \cite{finn2018meta}, which showed that MAML is a universal learning algorithm approximator under certain conditions. Then {\em MAML-type algorithms} have been studied recently from the optimization perspective, where the convergence rate and computation complexity is typically characterized. \cite{finn2019online} analyzed online MAML for a strongly convex objective function under a bounded-gradient assumption. 
 \cite{fallah2020convergence} developed a convergence analysis for one-step MAML for a general nonconvex objective in the  resampling case. Our study here provides a new convergence analysis for {\em multi-step} MAML in the {\em nonconvex} setting for both the resampling  and  finite-sum cases. 
 
Since the initial version of this manuscript was posted in arXiv, there have been a few studies on multi-step MAML more recently. \cite{wang2020global2} studied the global optimality of MAML, and characterized the  optimality gap of the stationary points. \cite{wang2020global} established the global convergence of MAML with over-parameterized deep neural networks. \cite{kim2020multi} 
proposed an efficient extension of multi-step MAML by gradient reuse in the inner loop. 
\cite{ji2020convergence} analyzed the convergence and complexity performance of multi-step ANIL algorithm, which is an efficient simplification of MAML by adapting only partial parameters in the inner loop. 
We emphasize that the study here is the first along the line of studies on multi-step MAML.
 
 \vspace{0.1cm}
 
{\em Another type of meta-learning algorithms} has also been studied as a bi-level optimization problem. \cite{rajeswaran2019meta} proposed a meta-regularization variant of MAML named iMAML via bilevel optimization, and analyzed its convergence by assuming  that the regularized empirical risk minimization problem in the inner optimization stage is strongly convex.  
\cite{likhosherstov2020ufo} studied the convergence properties of a class of first-order bilevel optimization algorithms.  


\vspace{0.2cm}

\noindent{\bf Statistical theory for meta-learning.}  
\cite{zhou2019efficient} statistically demonstrated the importance of prior hypothesis in reducing the excess risk via a regularization approach.  
\cite{du2020few} studied few-shot learning from a representation learning perspective, and showed that representation learning can provide a sufficient rate improvement in both linear regression and learning neural networks. \cite{tripuraneni2020provable} studied a multi-task linear regression problem with shared low-dimensional representation, and proposed a sample-efficient algorithm with performance guarantee. \cite{arora2020provable} proposed a representation learning approach for imitation learning via bilevel optimization, and demonstrated the improved sample complexity brought by representation learning. 



\section{Problem Setup}
In this paper, we study the convergence of the multi-step MAML algorithm. We consider two types of objective functions that are commonly used in practice: (a) {\bf resampling case}~\citep{finn2017model,fallah2020convergence}, where loss functions take the form  in expectation and new data are sampled as the algorithm runs; and (b) {\bf finite-sum case}~\citep{antoniou2019train}, where loss functions take the finite-sum form with given samples. The resampling case occurs often in reinforcement learning where data are continuously sampled as the algorithm iterates, whereas the finite-sum case typically occurs in classification problems where the datasets are already sampled in advance. In Appendix~\ref{ggpopssasdasdax}, we provide examples for these two types of problems. 

\subsection{Resampling Case: Problem Setup and Multi-Step MAML}
Suppose a set $\mathcal{T} = \{\mathcal{T}_i, i\in \mathcal{I}\}$ of tasks are available for learning and tasks are sampled based on a probability distribution $p(\mathcal{T})$ over the task set. Assume that each task $\mathcal{T}_i$  is associated with a loss  $l_i(w): \mathbb{R}^d \rightarrow \mathbb{R}$ parameterized by $w$.

The goal of multi-step MAML is to find a good initial parameter $w^*$ such that after observing a new task, a few gradient descend steps starting from such a point $w^*$ can efficiently approach the optimizer (or a stationary point) of the corresponding loss function. Towards this end, multi-step MAML consists of two nested stages, where the inner stage consists of  {\em multiple} steps of (stochastic) gradient descent for each individual tasks, and the outer stage updates the meta parameter over all the sampled tasks. More specifically, at each inner stage, each $\mathcal{T}_i$ initializes at the meta parameter, i.e., $\widetilde w^i_0 := w$, and runs $N$ {\em gradient descent} steps as 
\begin{align}\label{gd_w}
&\widetilde w^i_{j+1} = \widetilde w^i_j - \alpha \nabla l_i(\widetilde w^i_j), \quad j = 0,1,..., N-1 .
\end{align}
Thus, the loss of task $\mathcal{T}_i$ after the $N$-step inner stage iteration is given by $l_i(\widetilde w^i_N)$, where $\widetilde w^i_N$ depends on the meta parameter $w$ through the iteration updates in \eqref{gd_w}, and can hence be written as $\widetilde w^i_N(w)$.  We further define $\mathcal{L}_i(w):=l_i(\widetilde w^i_N(w))$, and hence the overall meta objective is given by
\begin{align}\label{objective}
\min_{w\in\mathbb{R}^d} \mathcal{L}(w):= \mathbb{E}_{i\sim p(\mathcal{T})}[ \mathcal{L}_i(w)]  :=  \mathbb{E}_{i\sim p(\mathcal{T})} [l_i(\widetilde w^i_N(w))].
\end{align}
Then the outer stage of meta update is a gradient decent step to optimize the above objective function. Using the chain rule, we provide a simplified form (see Appendix~\ref{simplifeid} for its derivations) of  gradient $\nabla \mathcal{L}_i(w)$ by   
	\begin{align}\label{nablaF}
	\nabla \mathcal{L}_i(w) = \bigg[ \prod_{j=0}^{N-1}(I-\alpha \nabla^2 l_i(\widetilde w^i_{j}))\bigg]\nabla l_i(\widetilde w^i _{N}),
	\end{align}
	where $\widetilde w^i_{0} = w$ for all tasks. 
	Hence, the {\em full gradient descent} step of the outer stage for~\eqref{objective} can be written as  
\begin{align}\label{true_meta_gd}
w_{k+1}  = w_k - \beta_k \mathbb{E}_{i\sim p(\mathcal{T})}\bigg[\prod_{j=0}^{N-1}(I-\alpha \nabla^2 l_i(\widetilde w^i_{k,j}))\bigg]\nabla l_i(\widetilde w^i _{k,N}),
\end{align}
where the index $k$ is added to $\widetilde w^i_{j}$ in \eqref{nablaF} to denote that these parameters are at the $k^{th}$ iteration of the meta parameter $w$.

\begin{algorithm}[t]
	\caption{Multi-step MAML in the resampling case} 
	\label{alg:online}
	\begin{algorithmic}[1]
		\STATE {\bfseries Input:}  Initial parameter $w_0$, inner stepsize $\alpha>0$	
		\WHILE{not done}
		\STATE{Sample $B_k\subset \mathcal{I}$ of i.i.d. tasks by distribution $p(\mathcal{T})$}
		\FOR{all tasks $\mathcal{T}_i$ in $B_k$}
		\FOR{$j = 0, 1,...,N-1$}
		\STATE{Sample a training set $S^i_{k,j}$ 
			\\ Update { $w^i_{k, j+1} = w^i_{k, j} - \alpha \nabla l_i(w^i_{k,j}; S^i_{k,j})$}
		}
		\ENDFOR
		\ENDFOR
		\STATE{Sample $T^i_k$ and $D_{k,j}^i$ and compute {\small$\widehat G_i(w_k)$} through~\eqref{eq:metagrad_est}.}
 	\STATE{update  
$		w_{k+1}= w_k - \beta_k\frac{\sum_{i\in B_k}\widehat G_i(w_k) }{|B_k|}$.
			\\ Update $k \leftarrow k+1$}
		\ENDWHILE
	\end{algorithmic}
	\end{algorithm}

The inner- and outer-stage updates of MAML given in \eqref{gd_w} and \eqref{true_meta_gd} involve the gradient $\nabla l_i(\cdot)$ 
 and the Hessian $\nabla^2 l_i(\cdot)$ of the loss function $l_i(\cdot)$, which takes the form of the expectation over the distribution of data samples as given by 
\begin{align}\label{fiw}
l_i(\cdot) = \mathbb{E}_{\tau} l_i(\cdot\,; \tau),
\end{align} 
where $\tau$ represents the data sample. In practice, these two quantities based on the population loss function are estimated by samples. In specific, each task $\mathcal{T}_i$ samples a batch $\Omega$ of data under the current parameter $w$, and uses $\nabla l_i(\cdot\,; \Omega):= \frac{\sum_{\tau \in \Omega} \nabla l_i(\cdot\,; \tau)}{|\Omega|} $  and $\nabla^2 l_i(\cdot\,; \Omega):= \frac{\sum_{\tau \in \Omega} \nabla^2 l_i(\cdot\,; \tau)}{|\Omega|} $ as {\em unbiased} estimates of  the gradient $\nabla l_i(\cdot)$ and the Hessian $\nabla^2 l_i(\cdot)$, respectively. 

\vspace{0.01cm}

For practical multi-step MAML as shown in Algorithm~\ref{alg:online}, at the $k^{th}$ outer stage, we sample a set $ B_k$ of tasks. Then,  at the inner stage, each task $\mathcal{T}_i\in B_k$ samples a training set {\small$S_{k,j}^i$} for each iteration $j$ in the inner stage, uses  {\small$\nabla l_i(w^i_{k,j};S^i_{k,j})$} as an estimate of {\small$\nabla l_i(\widetilde w^i_{k,j})$} in \eqref{gd_w}, and runs a SGD update as
\begin{align}\label{es:up}
w^i_{k, j+1} = w^i_{k, j} - \alpha \nabla l_i(w^i_{k,j};S^i_{k,j}), \quad  j=0,..,N-1,
\end{align} 
where the initialization parameter $w^i_{k,0}=w_k$ for all $i \in B_k$. 

At the $k^{th}$ outer stage, we draw a  batch {\small$T^i_k$} and {\small$D_{k,j}^i$} of data samples  independent from each other and both independent from {\small$S^i_{k,j}$} and
use {\small$\nabla l_i(w_{k,N}^i;T^i_k)$} and  {\small$ \nabla^2 l_i(w_{k,j}^i; D_{k,j}^i)$} to estimate {\small$\nabla l_i(\widetilde w^i _{k,N})$} and {\small$\nabla^2 l_i(\widetilde w^i_{k,j})$} in~\eqref{true_meta_gd}, respectively.
Then, the meta parameter $w_{k+1}$ at the outer stage is updated by a SGD step as shown in line $10$ of Algorithm~\ref{alg:online}, 
where the estimated gradient $\widehat G_i(w_k)$  has a form of 
\begin{align}\label{eq:metagrad_est}
\widehat G_i(w_k) =\prod_{j=0}^{N-1}\big(I - \alpha \nabla^2 l_i\big(w_{k,j}^i;D_{k,j}^i\big)\big)\nabla l_i(w_{k,N}^i; T^i_k).
\end{align}
For simplicity, we suppose the sizes of { $S_{k,j}^i$, $D_{k,j}^i$} and  { $T_k^i$}  are $S$, $D$ and $T$ in this paper.

\subsection{Finite-Sum Case: Problem Setup and Multi-Step MAML}\label{app:finitesum}
In the finite-sum case, each task $\mathcal{T}_i$ is {\em pre-assigned} with a support/training sample set $S_i$ and a query/test sample set $T_i$. Differently from the resampling case, these sample sets are fixed and no additional fresh data are sampled as the algorithm runs. The goal here is to learn an initial parameter $w$ such that for each task $i$, after $N$ {\em gradient descent} steps on data from $S_i$ starting from this $w$, we can find a parameter $w_N$ that performs well on the test data set $T_i$. Thus, each task $\mathcal{T}_i$ is associated with two fixed loss functions { $l_{S_i}(w):= \frac{1}{|S_i|}\sum_{\tau \in S_i} l_i(w; \tau)$} and { $l_{T_i}(w):= \frac{1}{|T_i|}\sum_{\tau \in T_i} l_i(w; \tau)$} with a finite-sum structure, where $l_i(w; \tau)$ is the loss on a single sample point $\tau$ and a parameter $w$.  Then, the meta objective function takes the form of 
\begin{align}\label{objective2}
\min_{w\in\mathbb{R}^d} \mathcal{L}(w):= \mathbb{E}_{i\sim p(\mathcal{T})} [\mathcal{L}_i(w)] =  \mathbb{E}_{i\sim p(\mathcal{T})} [l_{T_i}(\widetilde w^i_N)], 
\end{align}
where $\widetilde w^i_N$ is obtained by 
\begin{align}\label{innerfinite}
&\widetilde w^i_{j+1} = \widetilde w^i_j - \alpha \nabla l_{S_i}(\widetilde w^i_j),  \quad j = 0, 1,..., N-1 \, \text{ with }\, \widetilde w^i_0 := w.
\end{align} 
Similarly to the resampling case, we define  the expected losses  $l_S(w)=\mathbb{E}_{i} l_{S_i}(w)$  and $l_T(w)=\mathbb{E}_{i} l_{T_i}(w)$, and the meta gradient step of the outer stage for \eqref{objective2} can be written as 
\begin{align}\label{fulll_updd}
w_{k+1}  = w_k - \beta_k \mathbb{E}_{i\sim p(\mathcal{T})}\prod_{j=0}^{N-1}(I-\alpha \nabla^2 l_{S_i}(\widetilde w^i_{k,j}))\nabla l_{T_i}(\widetilde w^i _{k,N}),
\end{align}
where the index $k$ is added to $\widetilde w^i_{j}$ in \eqref{innerfinite} to denote that these parameters are at the $k^{th}$ iteration of the meta parameter $w$.
 
\begin{algorithm}[t]
	\caption{Multi-step MAML in the finite-sum case} 
	\label{alg:offline}
	\begin{algorithmic}[1]
		\STATE {\bfseries Input:}  Initial parameter $w_0$, inner stepsize $\alpha>0$	
		\WHILE{not done}
		\STATE{Sample  $B_k\subset \mathcal{I}$ of i.i.d. tasks  by distribution $p(\mathcal{T})$}
		\FOR{all tasks $\mathcal{T}_i$ in $ B_k$}
		\FOR{$j = 0, 1,...,N-1$}
		\STATE{
			Update { $w^i_{k, j+1} = w^i_{k, j} - \alpha \nabla l_{S_i}\big(w^i_{k,j}\big)$}
		}
		\ENDFOR
		\ENDFOR
		\STATE{
			Update $w_{k+1}= w_k -\frac{\beta_k}{|B_k|} \sum_{i\in B_k}\widehat G_i(w_k)$
			\\ Update $k = k+1$}
		\ENDWHILE
	\end{algorithmic}
\end{algorithm}

As shown in Algorithm~\ref{alg:offline}, MAML in the finite-sum case has a nested structure similar to that in the resampling case except that it does not sample fresh data at each iteration. 
In the inner stage, MAML performs a sequence of {\em full gradient descent steps} (instead of stochastic gradient steps as in the resampling case) for each task $i \in B_k$ given  by 
\begin{align}\label{giniteoo}
w^i_{k, j+1} = w^i_{k, j} - \alpha \nabla l_{S_i}\big(w^i_{k,j}\big), \text{ for } j=0,....,N-1
\end{align}  
where $w_{k,0}^i = w_k$ for all $i\in B_k$. As a result, the parameter $w_{k,j}$ (which denotes the parameter due to the full gradient update) in the update step~\eqref{giniteoo} is equal to $\widetilde w_{k,j}$ in \eqref{fulll_updd} for all $j=0,...,N$. 


At the outer-stage iteration, the meta optimization of MAML performs a SGD step as shown in line 9 of Algorithm~\ref{alg:offline},  where $ \widehat G_i(w_k) $ is given by 
\begin{align}\label{offline:obj}
 \widehat G_i(w_k) &=\prod_{j=0}^{N-1}(I - \alpha \nabla^2 l_{S_i}(w_{k,j}^i))\nabla l_{T_i}(w_{k,N}^i).
\end{align}

Compared with the resampling case, the biggest difference for analyzing Algorithm~\ref{alg:offline} in the finite-sum case is that the losses $ l_{S_i}(\cdot)$ and $ l_{T_i}(\cdot)$ used in the inner and outer stages respectively are different from each other,  whereas in the resampling case, they both are equal to $l_i(\cdot)$ which takes the expectation over the corresponding samples. 
Thus, the convergence analysis for the finite-sum case  requires to develop different techniques.  For simplicity, we assume that the sizes of all $B_k$ are $B$.

\section{Convergence of Multi-Step MAML in Resampling Case} \label{theory:online}
In this section, we first make some basic assumptions for the meta loss functions in Section~\ref{sub:assum}, and then describe several challenges in analyzing the multi-step MAML in Section~\ref{sec:challenge}, and then present several  properties of the meta gradient in Section~\ref{se:opro}, and finally provide the convergence and complexity results for multi-step MAML in Section~\ref{main:online}. 
\subsection{Basic Assumptions}\label{sub:assum}
We adopt the following standard assumptions 
\citep{fallah2020convergence,rajeswaran2019meta}.
\begin{Assumption}\label{assum:smooth}
	The loss $l_i(\cdot)$ of task $\mathcal{T}_i$ given by~\eqref{fiw} satisfies
	\begin{enumerate}
		\item The loss $l_i(\cdot)$  is bounded below, i.e., { $ \inf_{w\in\mathbb{R}^d} l_i(w) > -\infty$}.  \label{item1}		

		\item $\nabla l_i(\cdot)$ is $L_i$-Lipschitz, i.e., for any $w,u\in\mathbb{R}^d$,     \label{item2}
$		\|\nabla l_i(w)-\nabla l_i(u)\| \leq L_i\|w-u\|$.
		\item  $\nabla^2 l_i(\cdot)$ is  $\rho_i$-Lipschitz, i.e., for any $w,u\in\mathbb{R}^d$,
$		\|\nabla^2 l_i(w)-\nabla^2 l_i(u)\| \leq \rho_i\|w-u\|$.
	\end{enumerate}
\end{Assumption}

By the definition of the objective function $\mathcal{L}(\cdot)$ in~\eqref{objective},  item~\ref{item1} of  Assumption~\ref{assum:smooth} implies that $\mathcal{L}(\cdot)$ is bounded below. In addition, 
item~\ref{item2}  implies {\small$\|\nabla^2 l_i(w)\|\leq L_i$} for any $w\in\mathbb{R}^d$. 

For notational convenience,   we take $L=\max_i L_i$ and $\rho = \max_{i} \rho_i$. 
The following assumptions impose the  bounded-variance conditions on $\nabla l_i(w)$, $\nabla l_i(w; \tau)$ and $\nabla^2 l_i(w; \tau)$.
\begin{Assumption}\label{a2}
	The stochastic gradient $\nabla l_i(\cdot)$ (with $i$ uniformly randomly chosen from set $\mathcal{I}$) has bounded variance, i.e., there exists a constant $\sigma>0$ such that, for any $w\in\mathbb{R}^d$,
	\begin{align*}
	\mathbb{E}_{i}\|\nabla l_i(w) - \nabla l(w)\|^2 \leq \sigma^2, 
	\end{align*}
	where the expected loss function $l(w):=\mathbb{E}_{i}l_i(w)$.
\end{Assumption}

\begin{Assumption}\label{a3}
	For any $w\in\mathbb{R}^d$ and $i\in\mathcal{I}$, there exist constants $\sigma_g, \sigma_H>0$ such that  
	\begin{align*}
	&\mathbb{E}_{\tau}\|\nabla l_i(w; \tau)- \nabla l_i(w)\|^2 \leq \sigma_g^2\;\text{ and }\;
\mathbb{E}_{\tau}\|\nabla^2 l_i(w;\tau)- \nabla^2 l_i(w)\|^2 \leq \sigma_H^2.
	\end{align*}
\end{Assumption}

Note that the above assumptions are  made only on individual loss functions $l_i(\cdot)$ rather than on the total loss $\mathcal{L}(\cdot)$, because some  conditions do not hold for $\mathcal{L}(\cdot)$, as shown later. 

\subsection{Challenges of Analyzing Multi-Step MAML}\label{sec:challenge}

Several new challenges arise when we analyze the convergence of {\em multi}-step MAML (with $N\ge 2$) compared to the one-step case (with $N=1$).

First, each iteration of  the meta parameter affects the overall objective function via a nested structure of $N$-step SGD optimization paths over all tasks. 
Hence, our analysis of the convergence of such a meta parameter 
needs to characterize the nested structure and the recursive updates. 

Second, the meta gradient estimator  {\small$\widehat G_i(w_k)$}  given in \eqref{eq:metagrad_est} involves {\small$\nabla^2 l_i(w_{k,j}^i;D_{k,j}^i)$} for $ j=1,...,N-1$, which are all {\em biased} estimators of {\small$\nabla^2 l_i(\widetilde w^i_{k,j})$} in terms of the randomness over $D_{k,j}^i$.
This is because $ w^i_{k,j}$ is  a stochastic estimator of $\widetilde w^i_{k,j}$ obtained via  random training sets $S_{k,t}^i, t=0,...,j-1$ along an $N$-step SGD optimization path in the inner stage. In fact, such a bias error occurs only for multi-step MAML with $N\geq 2$ (which equals zero for $N=1$), and requires additional efforts to handle. 

Third, both the Hessian term {\small$\nabla^2 l_i(w_{k,j}^i; D_{k,j}^i)$} for $j=2,...,N-1$ and the gradient term {\small$\nabla l_i(w_{k,N}^i; T^i_k)$} in the meta gradient estimator {\small$\widehat G_i(w_k)$} given in \eqref{eq:metagrad_est} depend on the sample sets  $S_{k,i}^i$ used for inner stage iteration to obtain $w_{k,N}^i$, and hence they are statistically {\em correlated} even conditioned on $w_k$. Such complication also occurs only for multi-step MAML with $N\geq 2$ and requires new treatment (the two terms are independent for $N=1$).



\subsection{Properties of Meta Gradient}\label{se:opro}
Differently from the conventional  gradient whose corresponding loss is evaluated directly at the current parameter $w$, the meta  gradient has a more complicated nested structure with respect to $w$, because its loss is evaluated at the final output of the inner optimization stage, which is $N$-step SGD updates.
As a result, analyzing the meta gradient is very different and more challenging compared to analyzing the conventional  gradient. In this subsection, we establish some important properties of the meta gradient which are useful for characterizing the convergence of multi-step MAML.  


Recall that  $\nabla \mathcal{L}(w) = \mathbb{E}_{i\sim p(\mathcal{T})} [\nabla \mathcal{L}_i(w)]$ with  $\nabla \mathcal{L}_i(w)$ given by~\eqref{nablaF}. The following proposition characterizes the Lipschitz property of the gradient $\nabla \mathcal{L}(\cdot)$. 
\begin{proposition}\label{th:lipshiz}
	 Suppose that  Assumptions~\ref{assum:smooth},~\ref{a2} and~\ref{a3} hold. For any $w,u\in\mathbb{R}^d$, we have 
	\begin{align*}
	\|\nabla \mathcal{L}(w) - \nabla \mathcal{L}(u)\|  \leq \big( (1+\alpha L) ^{2N}L + C_\mathcal{L} \mathbb{E}_{i}\|\nabla l_i(w)\|\big) \|w-u\|,
	\end{align*}
	where $C_\mathcal{L}$ is a positive constant  given by{
	\begin{align}\label{clcl}
	C_\mathcal{L}= \big(  (1+\alpha L) ^{N-1}\alpha \rho + \frac{\rho}{L}  (1+\alpha L) ^N ( (1+\alpha L) ^{N-1} -1) \big)  (1+\alpha L) ^N.
	\end{align}}
\end{proposition} 
The proof of Proposition~\ref{th:lipshiz} handles the first challenge described in Section~\ref{sec:challenge}. More specifically,  
we bound the differences between $\widetilde w_j^i$ and $\widetilde u_j^i$ along two separate paths $(\widetilde w_j^i, j =0,....,N)$ and $(\widetilde u_j^i, j =0,....,N )$, and then connect these differences to the distance $\|w-u\|$. Proposition~\ref{th:lipshiz} shows that the objective $\mathcal{L}(\cdot)$ has a gradient-Lipschitz parameter $$L_w =  (1+\alpha L) ^{2N}L + C_\mathcal{L}\mathbb{E}_{i}\|\nabla l_i(w)\|,$$ 
which can  be unbounded due to the fact that $\nabla l_i(w)$ may be unbounded. 
Similarly to~\cite{fallah2020convergence}, we  use 
\begin{align}\label{hatlw}
\widehat L_{w_k} =   (1+\alpha L) ^{2N}L + \frac{ C_\mathcal{L}\sum_{i\in B_k^\prime}\|\nabla l_i(w_k; D_{L_k}^i)\|}{|B_k^\prime|}
\end{align}
to estimate $L_{w_k}$ at the meta parameter $w_k$, where we  {\em independently} sample the data sets $B_k^\prime$ and $D_{L_k}^i$. As will be shown in Theorem~\ref{th:mainonline}, we  set the meta stepsize $\beta_k$ to be inversely proportional to { $\widehat L_{w_k} $} to handle the  possibly unboundedness. In the experiments, we find that the gradients $\nabla l_i(w_k), k\geq 0$ are well bounded during the optimization process, and hence a constant outer-stage stepsize is sufficient in practice. 

We next characterize several estimation properties of  the meta gradient estimator  $\widehat G_i(w_k)$ in \eqref{eq:metagrad_est}. 
Here, we address the second and third challenges described in Section~\ref{sec:challenge}.  We first quantify how far  $w_{k,j}^i$ is away from $\widetilde w_{k,j}^i$, and then provide upper bounds on the first- and second-moment distances between $w_{k,j}^i$ and $\widetilde w_{k,j}^i$ for all { $j= 0,..., N$} as below. 
\begin{proposition}\label{le:distance}
	Suppose that  Assumptions~\ref{assum:smooth},~\ref{a2} and~\ref{a3} hold. Then, for any  $j=0,..., N$ and $i \in B_k$, we have
\begin{itemize}

	\item {\bf First-moment :} $\mathbb{E}(\|w_{k,j}^i - \widetilde w_{k,j}^i \| \, | w_k) \leq \big( (1+\alpha L) ^j -1 \big) \frac{\sigma_g }{L \sqrt{S}}$.

\item  {\bf Second-moment:} $\mathbb{E}(\|w_{k,j}^i - \widetilde w_{k,j}^i \|^2\, | w_k) \leq \big((1+\alpha L +2\alpha^2 L^2)^j -1 \big) \frac{\alpha \sigma_g ^2}{(1+\alpha L)L   S}$.
	\end{itemize}
\end{proposition}
Proposition~\ref{le:distance} shows that we can effectively upper-bound the point-wise distance between two paths  by choosing $\alpha$ and $S$ properly. Based on Proposition~\ref{le:distance}, we provide an upper bound on the first-moment estimation error of { $\widehat G_i(w_k)$}. 
\begin{proposition}\label{th:first-est} 
	Suppose   Assumptions~\ref{assum:smooth},~\ref{a2} and~\ref{a3} hold, and define constants 
	\begin{align}\label{ppolp}
  C_{{\text{\normalfont err}}_1} =  (1+\alpha L)^{2N} \sigma_g, \;\;C_{{\text{\normalfont err}}_2}  = \frac{ (1+\alpha L) ^{4N}\rho \sigma_g}{\big (2-  (1+\alpha L)^{2N}\big )L^2} .
	\end{align}
	Let $e_k := \mathbb{E}[\widehat G_i(w_k)] - \nabla \mathcal{L}(w_k) $
	 be the estimation error. If  the inner stepsize $\alpha < (2^{\frac{1}{2N}} - 1)/L$, then conditioning on $w_k$, we have 
	\begin{align}\label{es:original}
	\|e_k\| \leq \frac{C_{{\text{\normalfont err}}_1} }{\sqrt{S}} + \frac{C_{{\text{\normalfont err}}_2}  }{\sqrt{S}} (\|\nabla \mathcal{L}(w_k)\| + \sigma).
	\end{align} 
\end{proposition}
Note that 
the estimation error for the multi-step case shown in Proposition~\ref{th:first-est} involves a term $\mathcal{O}\big(\frac{\|\nabla \mathcal{L}(w_k)\|}{\sqrt{S}}\big)$, which cannot be avoided due to the Hessian approximation error caused by the randomness over the samples sets $S_{k,j}^i$. Somewhat interestingly, our later analysis shows that this term does not affect the final convergence rate if we choose the size $S$ properly. 
The following proposition provides an upper-bound on the second moment  of the meta gradient estimator $\widehat G_i(w_k)$.
\begin{proposition}\label{th:second} 
	Suppose that Assumptions \ref{assum:smooth}, \ref{a2} and \ref{a3} hold. Define  constants 
 \begin{align}\label{para:seq12}
	C_{\text{\normalfont squ}_1} &= 3\Big(  \frac{\alpha^2 \sigma_H^2}{D} + (1+\alpha L)^2  \Big)^N \sigma_g^2, 	\;\;
	C_{\text{\normalfont squ}_3} =\frac{2C_{\text{\normalfont squ}_1}  (1+\alpha L)^{2N}}{(2- (1+\alpha L)^{2N})^2\sigma_g^2}, \nonumber
	\\       C_{\text{\normalfont squ}_2}  &= C_{\text{\normalfont squ}_1} \big((1+2\alpha L +2\alpha^2L^2)^N-1\big)	\alpha L  (1+\alpha L)^{-1}.
	\end{align}
	If  the inner stepsize $\alpha < (2^{\frac{1}{2N}} - 1)/L$, then conditioning on $w_k$, we have 
	\begin{align}\label{bigmans}
	\mathbb{E}\|\widehat G_i(w_k)\|^2 \leq& \frac{	C_{\text{\normalfont squ}_1}}{T} + \frac{	C_{\text{\normalfont squ}_2} }{S} +	C_{\text{\normalfont squ}_3} \left(  \|\nabla \mathcal{L}(w_k)\|^2 +\sigma^2   \right).
	\end{align}
\end{proposition}
By choosing set sizes $D,T,S$ and the inner stepsize $\alpha$ properly, the factor $C_{\text{\normalfont squ}_3}$ in the second-moment error bound in \eqref{bigmans} can be made at a constant level and the first two error terms $\frac{C_{\text{\normalfont squ}_1}}{T} $ and $ \frac{C_{\text{\normalfont squ}_2} }{S}$ can be made sufficiently small so that the variance of the meta gradient estimator can be well controlled in the convergence analysis, as shown later. 
\subsection{Main Convergence Result}\label{main:online}
By using the properties of the meta gradient established in Section~\ref{se:opro}, we provide the convergence rate for multi-step MAML of Algorithm~\ref{alg:online} in the following theorem. 
\begin{theorem}\label{th:mainonline}  
	Suppose that  Assumptions~\ref{assum:smooth},~\ref{a2} and~\ref{a3} hold. 
	Set the meta stepsize $\beta_k = \frac{1}{C_\beta \widehat L_{w_k}} $ with  $\widehat L_{w_k}$  given by~\eqref{hatlw}, where  $|B_k^\prime| > \frac{4C^2_{\mathcal{L}}\sigma^2}{3(1+\alpha L)^{4N}L^2}$ and $|D_{L_k}^i| > \frac{64\sigma^2_g C_\mathcal{L}^2}{(1+\alpha L)^{4N}L^2}$ for all $i \in B_k^\prime$.  Define  $\chi =  \frac{(2- (1+\alpha L)^{2N}) (1+\alpha L)^{2N}L}{C_{\mathcal{L}}} + \sigma $ and 
	\begin{align}\label{para:com}
	\xi =& \frac{6}{C_\beta L}\big( \frac{1}{5} +\frac{2}{C_\beta}   \big)\big( C^2_{{\text{\normalfont err}}_1} +C^2_{{\text{\normalfont err}}_2}\sigma^2\big), \quad \phi = \frac{2}{C_\beta^2 L} \Big(  \frac{C_{\text{\normalfont squ}_1}}{T}  +  \frac{C_{\text{\normalfont squ}_2}}{S} + C_{\text{\normalfont squ}_3} \sigma^2 \Big) \nonumber
	\\\theta = &  \frac{2\big(2- (1+\alpha L)^{2N}\big)}{C_\beta C_{\mathcal{L}}}\Big(   \frac{1}{5} - \big(\frac{3}{5} + \frac{6}{C_\beta}\big)\frac{C^2_{{\text{\normalfont err}}_2} }{S} - \frac{C_{\text{\normalfont squ}_3}}{C_\beta B} - \frac{2}{C_\beta}\Big)  
	\end{align}
    where $C_{{\text{\normalfont err}}_1},C_{{\text{\normalfont err}}_2}$ are given in \eqref{ppolp} and $C_{\text{\normalfont squ}_1},C_{\text{\normalfont squ}_2},C_{\text{\normalfont squ}_3} $ are given in \eqref{para:seq12}. 
	Choose the inner stepsize $\alpha < (2^{\frac{1}{2N}} - 1)/L$, and choose $C_\beta, S$ and $B$ such that $\theta >0$. 
	Then, Algorithm~\ref{alg:online} finds a solution $w_{\zeta}$ such that 
	\begin{align}\label{c:result}
	\mathbb{E}\|\nabla \mathcal{L}(w_\zeta) \|  \leq &\frac{\Delta}{\theta }\frac{1}{K} +    \frac{\xi}{\theta}\frac{1}{S} +  \frac{\phi }{\theta}\frac{1}{B}  + \sqrt{\frac{\chi}{2} }\sqrt{\frac{\Delta}{\theta }\frac{1}{K}
	+    \frac{\xi}{\theta}\frac{1}{S} +  \frac{\phi }{\theta}\frac{1}{B}},
	\end{align}
	where $\Delta = \mathcal{L}(w_0) - \mathcal{L}^*$ with $\mathcal{L}^*=\inf_{w\in\mathbb{R}^d} \mathcal{L}(w)$.
\end{theorem}
The proof of Theorem~\ref{th:mainonline} (see Section~\ref{prop:meta_grad} for details) consists of four main steps: step $1$ of bounding an iterative meta update by the meta-gradient smoothness established by Proposition~\ref{th:lipshiz}; step $2$ of characterizing first-moment estimation error of the meta-gradient estimator $\widehat G_i(w_k)$ by Proposition~\ref{th:first-est}; step $3$ of characterizing second-moment estimation error of the meta-gradient estimator $\widehat G_i(w_k)$ by Proposition~\ref{th:second}; and step $4$ of combining steps 1-3, and telescoping to yield the convergence. 

In Theorem~\ref{th:mainonline}, the convergence rate given by  \eqref{c:result}  mainly contains three parts:
the first term $\frac{\Delta}{\theta }\frac{1}{K}$ indicates that the meta parameter converges sublinearly with the number $K$ of meta iterations, 
 the second term $ \frac{\xi}{\theta}\frac{1}{S}$ captures the estimation error of $\nabla l_i(w^i_{k,j};S^i_{k,j})$ for approximating the full gradient $\nabla l_i(w^i_{k,j})$ which can be made sufficiently small by choosing a large sample size $S$,   
 and the third term $\frac{\phi }{\theta}\frac{1}{B}$ captures the estimation error and variance of the stochastic meta gradient, 
 which can be made small by choosing large $B,T$ and $D$
  (note that $\phi$ is proportional to both $\frac{1}{T}$ and $\frac{1}{D}$).

Our analysis reveals several insights for the convergence of multi-step MAML as follows.
(a) To guarantee convergence, we require $\alpha L< 2^{\frac{1}{2N}} - 1$ (e.g., $\alpha=\Theta(\frac{1}{NL})$). Hence, if the number $N$ of inner gradient steps is large and $L$ is not small (e.g., for some RL problems), we need to choose a small inner stepsize  $\alpha$ so that  the last output of the inner stage has a {\em strong dependence} on the initialization (i.e., meta parameter), as also shown and explained in  \cite{rajeswaran2019meta}. (b) For problems with small Hessians such as many classification/regression problems~\citep{finn2017model}, $L$ (which is an upper bound on the spectral norm of Hessian matrices) is small, and hence we can choose a larger $\alpha$. This explains the empirical findings in~\cite{finn2017model,antoniou2019train}.

We next specify the selection of parameters  to simplify the convergence result in Theorem~\ref{th:mainonline} and derive the  complexity of Algorithm~\ref{alg:online} for finding an $\epsilon$-accurate  stationary point. 

\begin{corollary}\label{co:online}
	Under the setting of Theorem~\ref{th:mainonline}, choose $\alpha = \frac{1}{8NL}, C_\beta = 100$ and let batch sizes $S\geq \frac{15\rho^2\sigma_g^2}{L^4}$ and $D\geq \sigma_H^2 L^2$. Then we have 
	\begin{align*}
	\mathbb{E}\|\nabla \mathcal{L}(w_\zeta) \|  \leq &\mathcal{O} \Big(  \frac{1}{K} + \frac{\sigma_g^2(\sigma^2+1)}{S} + \frac{\sigma_g^2 +\sigma^2}{B} +\frac{\sigma^2_g}{TB} 
	\\&+ \sqrt{\sigma +1}\sqrt{\frac{1}{K} + \frac{\sigma_g^2(\sigma^2+1)}{S} + \frac{\sigma_g^2 +\sigma^2}{B}+\frac{\sigma^2_g}{TB}}\Big).
	\end{align*}
To achieve $\mathbb{E}\|\nabla \mathcal{L}(w_\zeta) \|<\epsilon$, Algorithm~\ref{alg:online} requires at most  $\mathcal{O}\big(\frac{1}{\epsilon^2}\big)$ iterations, and $\mathcal{O}(\frac{N}{\epsilon^4}+\frac{1}{\epsilon^{2}})$ gradient computations and $\mathcal{O}\big(\frac{N}{\epsilon^{2}}\big)$ Hessian computations per meta iteration. 
\end{corollary}
\vspace{-0.15cm}


Differently from the conventional SGD that requires a gradient complexity of { $\mathcal{O}(\frac{1}{\epsilon^{4}})$}, MAML requires a higher gradient complexity by a factor of { $\mathcal{O}(\frac{1}{\epsilon^{2}})$}, which is unavoidable because MAML requires  { $\mathcal{O}(\frac{1}{\epsilon^{2}})$}  tasks to achieve an $\epsilon$-accurate meta point, whereas SGD runs only over one task.  

 Corollary~\ref{co:online} shows that 
given a properly chosen inner stepsize, e.g., $\alpha = \Theta(\frac{1}{NL})$, MAML is guaranteed to converge 
 with both the gradient and the Hessian computation complexities growing only {\em linearly} with $N$. These results explain some empirical findings for MAML training in~\cite{rajeswaran2019meta}.  
The above results can also be obtained by using a larger stepsize such as  $\alpha = \Theta(c^{\frac{1}{N}}-1)/L> \Theta\big(\frac{1}{NL}\big)$with a certain constant $c>1$. 

\section{Convergence of Multi-Step MAML in Finite-Sum Case}\label{theory:offline}
In this section, we provide several properties of the meta gradient for the finite-sum case, and then analyze the convergence and complexity of Algorithm~\ref{alg:offline}.
\subsection{Basic Assumptions} 
We state several standard assumptions  for the analysis in the finite-sum case.  
\begin{Assumption}\label{assum:smoothoff}
	For each task $\mathcal{T}_i$, the loss functions $l_{S_i}(\cdot)$ and  $l_{T_i}(\cdot)$
	in~\eqref{objective2} satisfy
	\begin{enumerate}
		\item $l_{S_i}(\cdot), l_{T_i}(\cdot)$ are bounded below,  i.e., $ \inf_{w\in\mathbb{R}^d} l_{S_i}(w) > -\infty$ and $\inf_{w\in\mathbb{R}^d} l_{T_i}(w) > -\infty$.
		\item Gradients $\nabla l_{S_i}(\cdot)$ and  $\nabla l_{T_i}(\cdot)$ are $L$-Lipschitz continuous, i.e.,  for any $w,u\in\mathbb{R}^d$
		\begin{align*}
		\|\nabla l_{S_i}(w)-\nabla l_{S_i}(u)\| \leq L\|w-u\| \text{ and }
		\|\nabla l_{T_i}(w)-\nabla l_{T_i}(u)\| \leq L\|w-u\|.
		\end{align*}
		\item  Hessians $\nabla^2 l_{S_i}(\cdot)$ and  $\nabla^2 l_{T_i}(\cdot)$ are  $\rho$-Lipschitz continuous, i.e., for any $w,u\in\mathbb{R}^d$
		\begin{align*}
		\|\nabla^2l_{S_i}(w)-\nabla^2l_{S_i}(u)\| \leq \rho\|w-u\| \text{ and }	\|\nabla^2l_{T_i}(w)-\nabla^2l_{T_i}(u)\| \leq \rho\|w-u\|.
		\end{align*}
	\end{enumerate}
\end{Assumption}
The following assumption provides  two conditions   $\nabla l_{S_i}(\cdot)$ and  $\nabla l_{T_i}(\cdot)$. 
\begin{Assumption}\label{assum:vaoff}
	For all $w\in\mathbb{R}^d$, gradients $\nabla l_{S_i}(w)$ and  $\nabla l_{T_i}(w)$ satisfy
	\begin{enumerate}
		\item $\nabla l_{T_i}(\cdot)$ has a bounded variance, i.e., there exists a constant $\sigma>0$ such that 
		\begin{align*}
		\mathbb{E}_{i}\|\nabla  l_{T_i}(w) - \nabla  l_{T}(w)\|^2 \leq \sigma^2,
		\end{align*}
		where $\nabla l_{T}(\cdot) =\mathbb{E}_{i} \left [ \nabla l_{T_i}(\cdot)\right]$. 
		\item For each $i\in \mathcal{I}$, there exists a constant $b_i>0$ such that $\|\nabla l_{S_i}(w)-\nabla l_{T_i}(w)\| \leq b_i.$
	\end{enumerate}
\end{Assumption}
Instead of  imposing a  bounded variance condition on the stochastic gradient $\nabla l_{S_i}(w)$, we alternatively assume the difference $\|\nabla l_{S_i}(w)-\nabla  l_{T_i}(w)\|$ to be upper-bounded by  a constant, which is more reasonable because sample sets $S_i$ and $T_i$ are often sampled from the same distribution and share certain statistical similarity. We note that the second condition also implies $\|\nabla l_{S_i}(w)\| \leq \| \nabla l_{T_i}(w)\|+ b_i$, which  is weaker than the bounded gradient assumption made in papers such as~\cite{finn2019online}. 
It is worthwhile mentioning that the second condition can be relaxed to $\|\nabla l_{S_i}(w)\| \leq c_i\| \nabla l_{T_i}(w)\|+ b_i$ for a constant $c_i>0$. Without the loss of generality, we consider $c_i=1$ for simplicity.

\subsection{Properties of Meta Gradient}\label{mainsec:off}
We develop several important properties of the meta gradient. 
The following proposition characterizes a Lipschitz property of the gradient  of the objective function   $$\nabla \mathcal{L}(w) =  \mathbb{E}_{i\sim p(\mathcal{T})} \prod_{j=0}^{N-1}(I - \alpha \nabla^2 l_{S_i}(\widetilde w_{j}^i))\nabla l_{T_i}(\widetilde w_{N}^i),$$ where the weights $ \widetilde w_{j}^i, i\in \mathcal{I}, j=0,..., N$ are given by the gradient descent steps in~\eqref{innerfinite}. 
\begin{proposition}\label{finite:lip} 
	Suppose that Assumptions~\ref{assum:smoothoff} and~\ref{assum:vaoff} hold.  Then, for any $w,u\in\mathbb{R}^d$, we have 
	\begin{align*}
	\|\nabla \mathcal{L}(w) - \nabla \mathcal{L}(u)\| \leq L_{w}\|w-u\|,\; L_{w} = (1+\alpha L)^{2N}L + C_bb + C_{\mathcal{L}} \mathbb{E}_{i}\|\nabla l_{T_i}(w)\|
	\end{align*}
	where $b=\mathbb{E}_{i} [b_i]$ and $C_b, C_{\mathcal{L}} >0$ are constants given by  
	\begin{small}
	\begin{align}\label{cl1ss}
	C_b= \big( \alpha \rho + \frac{\rho}{L} (1+\alpha L)^{N-1}  \big)(1+\alpha L)^{2N}, \;C_{\mathcal{L}} &= \big( \alpha \rho + \frac{\rho}{L}  (1+\alpha L)^{N-1} \big) (1+\alpha L)^{2N}.
	\end{align}
	\end{small}
\end{proposition}
Proposition~\ref{finite:lip} shows that $\nabla \mathcal{L}(w)$ has a Lipschitz parameter $L_{w}$. Similarly to~\eqref{hatlw}, we use the following construction 
\begin{align}\label{hlwkoff}
\hat L_{w_k} =(1+\alpha L)^{2N}L + C_b b +  \frac{C_\mathcal{L}}{|B_k^\prime|}\sum_{i \in B_k^\prime}\|\nabla l_{T_i}(w_k)\|,
\end{align}
at the $k^{th}$ outer-stage iteration to approximate $L_{w_k}$, where $B_k^\prime \subset \mathcal{I}$ is  chosen independently from $B_k$. 
It can be verified that the gradient estimator $\widehat G_i(w_k)$ given in~\eqref{offline:obj} is an {\em unbiased} estimate of  $\nabla \mathcal{L}(w_k)$. Thus, our next step is to upper-bound the second moment of $\widehat G_i(w_k)$.
\begin{proposition}\label{finite:seconderr} 
	Suppose Assumptions~\ref{assum:smoothoff} and~\ref{assum:vaoff} are hold, and define constants 
\begin{small}	
\begin{align}\label{wocaopp}
A_{\text{\normalfont squ}_1} = \frac{4(1+\alpha L)^{4N}}{(2-(1+\alpha L)^{2N})^2}, \;\;A_{\text{\normalfont squ}_2} = \frac{4(1+\alpha L)^{8N}}{(2-(1+\alpha L)^{2N})^2}(\sigma+b)^2 + 2(1+\alpha)^{4N}(\sigma^2+\widetilde b),
	\end{align}
	\end{small}
	\hspace{-0.15cm}where  $\widetilde b =\mathbb{E}_{i\sim p(\mathcal{T})}[b_i^2]$. Then, if $\alpha < (2^{\frac{1}{2N}} - 1)/L$, then conditioning on $w_k$, we have 
	\begin{small}
	\begin{align*}
	\mathbb{E}\|\widehat G_i(w_k)\|^2 \leq A_{\text{\normalfont squ}_1} \|\nabla \mathcal{L}(w_k)\|^2 + A_{\text{\normalfont squ}_2}.
	\end{align*}
	\end{small}
\end{proposition}
Based on the above properties, we next characterize the convergence  of multi-step MAML. 
\subsection{Main Convergence Results}\label{main:offline}
In this subsection, we provide the convergence and complexity  analysis for Algorithm~\ref{alg:offline} based on the properties established in the previous subsection. 
\begin{theorem}\label{mainth:offline} 
	Let Assumptions~\ref{assum:smoothoff} and~\ref{assum:vaoff} hold, and apply Algorithm~\ref{alg:offline} to solve the objective function~\eqref{objective2}. 
	Choose the meta stepsize $\beta_k = \frac{1}{C_\beta \widehat L_{w_k}} $ with  $\widehat L_{w_k}$  given by \eqref{hlwkoff}, where $C_\beta>0$ is a constant and the batch size $|B_k^\prime| $ satisfies $|B_k^\prime| \geq \frac{2C^2_\mathcal{L}\sigma^2}{( C_b b + (1+\alpha L)^{2N} L)^2}$.
	Define constants 
	\begin{small}
	\begin{align}\label{offline:constants}
	\xi = &\frac{2-(1+\alpha L)^{2N}}{C_\mathcal{L}} (1+\alpha L)^{2N}L+\frac{ \big(2-(1+\alpha L)^{2N}\big) C_b b }{C_\mathcal{L}} +(1+\alpha L)^{3N}b, \nonumber
	\\ 
	\theta = &	\frac{2-(1+\alpha L)^{2N}}{C_\mathcal{L}} \Big(  \frac{1}{C_\beta}  - \frac{1}{C_\beta^2}\Big( \frac{A_{\text{\normalfont squ}_1}}{B}+1\Big)\Big), \; \phi = \frac{A_{\text{\normalfont squ}_2}}{LC_\beta^2}
	\end{align}
	\end{small}
	\hspace{-0.15cm}where $C_b,C_{\mathcal{L}}, A_{\text{\normalfont squ}_1}$ and $A_{\text{\normalfont squ}_1}$ are given by \eqref{cl1ss} and \eqref{wocaopp}. Choose  $\alpha < (2^{\frac{1}{2N}} - 1)/L$, and choose $C_\beta$ and $B$ such that $\theta >0$. Then,  Algorithm~\ref{alg:offline} attains a solution $w_{\zeta}$  such that 
	\begin{align}\label{iopnn}
	\mathbb{E}\|\nabla \mathcal{L}(w_\zeta)\| \leq \frac{\Delta}{2\theta K} +\frac{\phi}{2\theta B} + \sqrt{ \xi \Big(\frac{\Delta}{\theta K} +\frac{\phi}{\theta B}\Big) + \Big(\frac{\Delta}{2\theta K} +\frac{\phi}{2\theta B}\Big)^2 }.
	\end{align}
\end{theorem}
The parameters $\theta, \phi$ and $\xi$ in Theorem~\ref{mainth:offline} take complicate forms. 
The following corollary specifies the parameters $C_\beta, \alpha$ in Theorem~\ref{mainth:offline}, and 
provides a simplified result for Algorithm~\ref{alg:offline}.
\begin{corollary}\label{co:mainoffline}
	Under the same setting of Theorem~\ref{mainth:offline}, choose  $\alpha = \frac{1}{8NL}, C_\beta = 80$. We have
	\begin{align*}
	\mathbb{E}\|\nabla \mathcal{L}(w_\zeta)\|  \leq  \mathcal{O}\Big(  \frac{1}{K} +\frac{\sigma^2}{B} +\sqrt{\frac{1}{K} +\frac{\sigma^2}{B} }  \Big).
	\end{align*}
In addition, suppose the  batch size $B$ further satisfies $B\geq C_B\sigma^2\epsilon^{-2}$, where $C_B$ is a sufficiently large constant. Then, to achieve an $\epsilon$-approximate stationary point,  Algorithm~\ref{alg:offline} requires at most $K=\mathcal{O}(\epsilon^{-2})$ iterations, and a total number $\mathcal{O}\big((T+NS)\epsilon^{-2}\big)$ of gradient computations and a number $\mathcal{O}\big(NS\epsilon^{-2}\big)$ of Hessian computations per iteration, where $T$ and $S$ correspond to the sample sizes of the pre-assigned sets $T_i,i\in\mathcal{I}$ and $S_i,i\in\mathcal{I}$.
\end{corollary}

\section{Proofs of Main Results}
In this section, we provide the proofs the main results for MAML in the resampling case and the finite-sum case, respectively. This section is organized as follows. 

For the resampling case, Section~\ref{prop:meta_grad} provides the proofs for the convergence properties of multi-step MAML in the {\em resampling case}, which include Propositions~\ref{th:lipshiz},~\ref{le:distance}, \ref{th:first-est}, \ref{th:second} on the properties of meta gradient, and Theorem~\ref{th:mainonline} and Corollary~\ref{co:online} on the convergence and complexity performance of multi-step MAML. 
The proofs of these results require several technical lemmas, which we relegate to the Appendix~\ref{aux:lemma}.  

Next, for the finite-sum case, Section~\ref{proof:meta_finite} provides the proofs for the convergence properties of multi-step MAML in the {\em finite-sum case}, which include 
 Propositions~\ref{finite:lip},~\ref{finite:seconderr} on the properties of meta gradient, and Theorem~\ref{mainth:offline} and  Corollary~\ref{co:mainoffline} on the convergence and complexity of multi-step MAML. The proofs of these results rely on several technical lemmas, which we relegate to the Appendix~\ref{aux:lemma_finite}.



\subsection{Proofs for Section~\ref{theory:online}: Convergence of Multi-Step MAML in Resampling Case}\label{prop:meta_grad}

To simplify notations, we let $\bar S^i_j$ and  $\bar D^i_j$ denote the randomness over $S_{k,m}^i, D_{k,m}^i,m=0,...,j-1$ and  let $\bar S_j$ and $\bar D_j$ denote all randomness over $\bar S^i_j, \bar D^i_j, i\in \mathcal{I}$, respectively.

\vspace{0.2cm}
\subsection*{Proof of Proposition~\ref{th:lipshiz}}
	First recall that $\nabla \mathcal{L}_i(w) =   \prod_{j=0}^{N-1}(I-\alpha \nabla^2 l_i(\widetilde w^i_{j}))\nabla l_i( \widetilde w^i _{N})$. Then, we have 
	\begin{align}\label{delff}
	\|\nabla \mathcal{L}_i(w) - \nabla \mathcal{L}_i(u)\| 
	 \leq&  \Big\| \prod_{j=0}^{N-1}(I-\alpha \nabla^2 l_i(\widetilde w^i_{j})) -  \prod_{j=0}^{N-1}(I-\alpha \nabla^2 l_i(\widetilde u^i_{j}))\Big\|\big\|\nabla l_i( \widetilde w^i _{N})\big\|   \nonumber
	\\ &+ (1+\alpha L)^N \|\nabla l_i( \widetilde w^i _{N})  - \nabla l_i( \widetilde u^i _{N}) \| \nonumber
	\\ \overset{(i)}\leq&  \Big\| \prod_{j=0}^{N-1}(I-\alpha \nabla^2 l_i(\widetilde w^i_{j})) -  \prod_{j=0}^{N-1}(I-\alpha \nabla^2 l_i(\widetilde u^i_{j}))\Big\|(1+\alpha L)^N\big\|\nabla l_i( w)\big\|   \nonumber
	\\ &+ (1+\alpha L)^N L \|\widetilde w^i _{N} - \widetilde u^i _{N} \| \nonumber
	\\ \overset{(ii)}\leq&  \underbrace{\Big\| \prod_{j=0}^{N-1}(I-\alpha \nabla^2 l_i(\widetilde w^i_{j})) -  \prod_{j=0}^{N-1}(I-\alpha \nabla^2 l_i(\widetilde u^i_{j}))\Big\|}_{V(N)}(1+\alpha L)^N\big\|\nabla l_i( w)\big\|   \nonumber
	\\ &+ (1+\alpha L)^{2N} L \|w-u\|,
	\end{align}
	where (i) follows from Lemma~\ref{le:jiw}, and (ii) follows from Lemma~\ref{d_u_w}.
	We next upper-bound the term $V(N)$ in the above inequality. 
	Specifically, define a more general quantity $V(m)$ by replacing $N$ in $V(N)$ with $m$.
	 Then, we have 
	\begin{align}\label{youyitian}
	V(m) 
	\leq& \Big\| \prod_{j=0}^{m-2}(I-\alpha \nabla^2 l_i(\widetilde w^i_{j})) \Big\|\big\|\alpha \nabla^2 l_i(\widetilde w_{m-1}^i) -\alpha \nabla^2 l_i(\widetilde u_{m-1}^i)  \big\| \nonumber
	\\&+ \Big\|\prod_{j=0}^{m-2}(I-\alpha \nabla^2 l_i(\widetilde w^i_{j})) - \prod_{j=0}^{m-2}(I-\alpha \nabla^2 l_i(\widetilde u^i_{j}))\Big\| \big\|I-\alpha\nabla^2 l_i(\widetilde u_{m-1}^i) \big\| \nonumber
	\\\leq& (1+\alpha L)^{m-1} \big\|\alpha \nabla^2 l_i(\widetilde w_{m-1}^i) -\alpha \nabla^2 l_i(\widetilde u_{m-1}^i)  \big\| \nonumber
	\\&+ (1+\alpha L)\Big\|\prod_{j=0}^{m-2}(I-\alpha \nabla^2 l_i(\widetilde w^i_{j})) - \prod_{j=0}^{m-2}(I-\alpha \nabla^2 l_i(\widetilde u^i_{j}))\Big\| \nonumber
	\\\leq & (1+\alpha L)^{m-1} \alpha \rho \|\widetilde w_{m-1}^i - \widetilde u_{m-1}^i\| + (1+\alpha L)V(m-1) \nonumber
	\\\leq & (1+\alpha L)^{m-1} \alpha \rho (1+\alpha L)^{m-1} \|w-u\| + (1+\alpha L)V(m-1). 
	\end{align}
	Telescoping~\eqref{youyitian} over $m$ from $1$ to $N$ and noting $V(1) \leq \alpha \rho\|w-u\|$, we have 
	\begin{align}\label{vbbn}
	V(N)&\leq (1+\alpha L)^{N-1}V(1) + \sum_{m=0}^{N-2}\alpha \rho (1+\alpha L)^{2(N-m)-2}\|w-u\|(1+\alpha L)^m \nonumber
	\\& =  (1+\alpha L)^{N-1}\alpha \rho \|w-u\| +\alpha \rho (1+\alpha L)^N\sum_{m=0}^{N-2} (1+\alpha L)^{m}\|w-u\| \nonumber
	\\& \leq \left( (1+\alpha L)^{N-1}\alpha \rho + \frac{\rho}{L} (1+\alpha L)^N ((1+\alpha L)^{N-1}-1)\right) \|w-u\|.
	\end{align}
	Recalling the definition of $C_\mathcal{L}$ and 
	 Combining~\eqref{delff},~\eqref{vbbn}, we have 
	\begin{align*}
	\|\nabla \mathcal{L}_i(w) - \nabla \mathcal{L}_i(u)\| \leq \big( C_\mathcal{L} \|\nabla l_i(w)\| + (1+\alpha L)^{2N}L  \big) \|w-u\|.
	\end{align*}
	Based on the above inequality, we have
	\begin{align*}
	\|\nabla \mathcal{L}(w) - \nabla \mathcal{L}(u)\| &= \|\mathbb{E}_{i\sim p(\mathcal{T})}(\nabla \mathcal{L}_i(w) - \nabla \mathcal{L}_i(u))\| \nonumber
	\\&\leq  \mathbb{E}_{i\sim p(\mathcal{T})}\|(\nabla \mathcal{L}_i(w) - \nabla \mathcal{L}_i(u))\|
	\\&\leq  \big( C_\mathcal{L}  \mathbb{E}_{i\sim p(\mathcal{T})}\|\nabla l_i(w)\| + (1+\alpha L)^{2N}L  \big) \|w-u\|,
	\end{align*}
	which finishes the proof. 

\subsection*{Proof of Proposition~\ref{le:distance}}	
	We first prove the first-moment bound. 
	Conditioning on $w_k$, we have 
	\begin{align}
	\mathbb{E}_{\bar S^i_m}\|w_{k,m}^i - \widetilde w_{k,m}^i\| \overset{(i)}=& \mathbb{E}_{\bar S^i_m}\big\|w_{k,m-1}^i - \alpha \nabla l_i(w_{k,m-1}^i; S_{k,m-1}^i) - (\widetilde w_{k,m-1}^i - \alpha \nabla l_i(\widetilde w_{k,m-1}^i) )\big\|  \nonumber
	\\ \leq& \mathbb{E}_{\bar S^i_m} \|w_{k,m-1}^i - \widetilde w_{k,m-1}^i\| + \alpha \mathbb{E}_{\bar S^i_m}\big\|\nabla l_i(w_{k,m-1}^i; S_{k,m-1}^i) - \nabla l_i(w_{k,m-1}^i)\big\|\nonumber
	\\ &+ \alpha \mathbb{E}_{\bar S^i_m}\big\| \nabla l_i(w_{k,m-1}^i) - \nabla l_i(\widetilde w_{k,m-1}^i) \big\| \nonumber
	\\ \leq & \alpha \mathbb{E}_{\bar S^i_{m-2}} \Big(      \mathbb{E}_{S_{k,m-1}^i} \big(\|\nabla l_i(w_{k,m-1}^i; S_{k,m-1}^i) - \nabla l_i(w_{k,m-1}^i)\big\| \,\Big | \bar S^i_{m-2}\big)\Big)    \nonumber
	\\ & + (1+\alpha L)  \mathbb{E}_{\bar S^i_{m-1}} \|w_{k,m-1}^i - \widetilde w_{k,m-1}^i\| \nonumber
	\\\overset{(ii)}\leq & (1+\alpha L)  \mathbb{E}_{\bar S^i_{m-1}} \|w_{k,m-1}^i - \widetilde w_{k,m-1}^i\|  +  \alpha\frac{\sigma_g}{\sqrt{S}},\nonumber
	\end{align}
	where (i) follows from~\eqref{gd_w} and~\eqref{es:up}, and (ii) follows from Assumption~\ref{a3}. 
	Telescoping the above inequality over $m$ from $1$ to $j$ and using the fact that $w_{k,0}^i = \widetilde w_{k,0}^i = w_k$, we have 
	\begin{align*}
	\mathbb{E}_{\bar S^i_j}\|w_{k,j}^i - \widetilde w_{k,j}^i\|  \leq ((1+\alpha L)^j-1) \frac{\sigma_g}{L\sqrt{S}},
	\end{align*}
	which finishes the proof of the first-moment bound. 
	We next begin to prove the second-moment bound. 
	Conditioning on $w_k$, we have 
	\begin{align*}
	\mathbb{E}_{\bar S^i_m}&\|w_{k,m}^i - \widetilde w_{k,m}^i\|^2 
	\\ = &  \mathbb{E}_{\bar S^i_{m-1}}\|w_{k,m-1}^i - \widetilde w_{k,m-1}^i\|^2  + 
	\alpha^2\mathbb{E}_{\bar S^i_m}\|\nabla l_i(w_{k,m-1}^i; S_{k,m-1}^i)- \nabla l_i(\widetilde w_{k,m-1}^i)\|^2
	\\ & -2\alpha\mathbb{E}_{\bar S^i_{m-1}}\left(\mathbb{E}_{S_{k,m-1}^i} \langle w_{k,m-1}^i - \widetilde w_{k,m-1}^i, \nabla l_i(w_{k,m-1}^i; S_{k,m-1}^i)- \nabla l_i(\widetilde w_{k,m-1}^i)\rangle \big | \bar S^i_{m-1}\right)
	\\ \overset{(i)}\leq &  \mathbb{E}_{\bar S^i_{m-1}}\|w_{k,m-1}^i - \widetilde w_{k,m-1}^i\|^2  -2\alpha\mathbb{E}_{\bar S^i_{m-1}} \langle w_{k,m-1}^i - \widetilde w_{k,m-1}^i, \nabla l_i(w_{k,m-1}^i)- \nabla l_i(\widetilde w_{k,m-1}^i)\rangle
	\\ & +
	\alpha^2\mathbb{E}_{\bar S^i_m}\left(  2\|\nabla l_i(w_{k,m-1}^i; S_{k,m-1}^i)- \nabla l_i( w_{k,m-1}^i)\|^2 + 2\|\nabla l_i( w_{k,m-1}^i)- \nabla l_i(\widetilde w_{k,m-1}^i)\|^2 \right)
	\\ \overset{(ii)}\leq &  \mathbb{E}_{\bar S^i_{m-1}}\|w_{k,m-1}^i - \widetilde w_{k,m-1}^i\|^2  +2\alpha\mathbb{E}_{\bar S^i_{m-1}} \| w_{k,m-1}^i - \widetilde w_{k,m-1}^i\|\|\nabla l_i(w_{k,m-1}^i)- \nabla l_i(\widetilde w_{k,m-1}^i))\|
	\\ & +
	\alpha^2\mathbb{E}_{\bar S^i_m}\left(  2\|\nabla l_i(w_{k,m-1}^i; S_{k,m-1}^i)- \nabla l_i( w_{k,m-1}^i)\|^2 + 2\|\nabla l_i( w_{k,m-1}^i)- \nabla l_i(\widetilde w_{k,m-1}^i)\|^2 \right)
	\\ \leq &  \mathbb{E}_{\bar S^i_{m-1}}\|w_{k,m-1}^i - \widetilde w_{k,m-1}^i\|^2  +2\alpha L\mathbb{E}_{\bar S^i_{m-1}} \| w_{k,m-1}^i - \widetilde w_{k,m-1}^i\|^2
	\\ & +
	2\alpha^2\mathbb{E}_{\bar S^i_{m-1}}\Big(  \frac{\sigma_g^2}{S}+ L^2\|w_{k,m-1}^i-  \widetilde w_{k,m-1}^i\|^2 \Big) 
	\\ \leq& \big(1+2\alpha L+2\alpha^2 L^2\big)  \mathbb{E}_{\bar S^i_{m-1}}\|w_{k,m-1}^i - \widetilde w_{k,m-1}^i\|^2   + \frac{2\alpha^2\sigma_g^2}{S},
	\end{align*}
	where (i) follows from $\mathbb{E}_{S_{k,m-1}^i}  \nabla l_i(w_{k,m-1}^i; S_{k,m-1}^i)= \nabla l_i(w_{k,m-1}^i)$ and (ii) follows from the inequality that $-\langle a,b\rangle\leq \|a\|\|b\|$ for any vectors $a,b$. 
 	Noting that $w_{k,0}^i= \widetilde w_{k,0}^i = w_k$ and telescoping the above inequality over $m$ from $1$ to $j$, we obtain
	\begin{align*}
	\mathbb{E}_{\bar S^i_j}\|w_{k,j}^i - \widetilde w_{k,j}^i\|^2 \leq \left( (1+2\alpha L + 2\alpha^2L^2) ^j -1 \right)\frac{\alpha \sigma_g^2}{L(1+\alpha L) S}.
	\end{align*}
	Then,taking the expectation over $w_k$ in the above inequality finishes the proof. 

\subsection*{Proof of Proposition~\ref{th:first-est}}
	Recall the definition that {$$\widehat G_i(w_k)=  \prod_{j=0}^{N-1}(I - \alpha \nabla^2 l_i(w_{k,j}^i; D_{k,j}^i))\nabla l_i(w_{k,N}^i; T^i_k).$$}
	Then, conditioning on $w_k$, we have
	\begin{align}\label{gmeans}
	\mathbb{E} \widehat G_i(w_k) =& \mathbb{E}_{\bar S_N, i\sim p(\mathcal{T})} \mathbb{E}_{\bar D_N}\Big(  \prod_{j=0}^{N-1}  \big(I - \alpha \nabla^2 l_i(w_{k,j}^i; D_{k,j}^i)\big) \mathbb{E}_{T_k^i} \nabla l_i(w_{k,N}^i;  T_k^i) \big | \bar S_N, i  \Big) \nonumber
	\\ = & \mathbb{E}_{\bar S_N, i\sim p(\mathcal{T})}   \prod_{j=0}^{N-1}  \mathbb{E}_{D_{k,j}^i}\big(I - \alpha \nabla^2 l_i(w_{k,j}^i; D_{k,j}^i)\big |\bar  S_N, i \big)  \nabla l_i(w_{k,N}^i)     \nonumber
	\\ = &  \mathbb{E}_{\bar S_N, i\sim p(\mathcal{T})}  \prod_{j=0}^{N-1}  \big(I - \alpha \nabla^2 l_i(w_{k,j}^i)\big)  \nabla l_i(w_{k,N}^i)    ,
	\end{align} 
	which, combined with {\small $\nabla \mathcal{L}(w_k)  =\mathbb{E}_{i\sim p(\mathcal{T})} \prod_{j=0}^{N-1}(I-\alpha \nabla^2 l_i(\widetilde w^i_{k,j}))\nabla l_i(\widetilde w^i _{k,N})$}, yields
		\begin{align}\label{eq:ek}
	\|\mathbb{E} \widehat G_i(w_k) & - \nabla \mathcal{L}(w_k)\|  \nonumber
	\\ \overset{(i)}\leq & \mathbb{E}_{\bar S_N,  i\sim p(\mathcal{T})} \Big \|  \prod_{j=0}^{N-1}  \big(I - \alpha \nabla^2 l_i(w_{k,j}^i)\big)  \nabla l_i(w_{k,N}^i)  -   \prod_{j=0}^{N-1}(I-\alpha \nabla^2 l_i(\widetilde w^i_{k,j}))\nabla l_i(\widetilde w^i _{k,N})  \Big  \|  \nonumber
	\\ \leq &  \mathbb{E}_{\bar S_N,  i\sim p(\mathcal{T})} \Big \|  \prod_{j=0}^{N-1}  \big(I - \alpha \nabla^2 l_i(w_{k,j}^i)\big)  \nabla l_i(w_{k,N}^i)  -   \prod_{j=0}^{N-1}(I-\alpha \nabla^2 l_i(w^i_{k,j}))\nabla l_i(\widetilde w^i _{k,N})  \Big  \| \nonumber
	\\ \leq &   \mathbb{E}_{\bar S_N,  i} \Big \|  \prod_{j=0}^{N-1}  \big(I - \alpha \nabla^2 l_i(w_{k,j}^i)\big) -   \prod_{j=0}^{N-1}(I-\alpha \nabla^2 l_i(\widetilde w^i_{k,j})) \Big  \| \big\| \nabla l_i(\widetilde w_{k,N}^i) \big \|     
	 \nonumber
	\\ &+     (1+\alpha L)^N \mathbb{E}_{\bar S_N,  i}\Big \|    \nabla l_i(w_{k,N}^i)  -  \nabla l_i(\widetilde w^i _{k,N})  \Big  \|  \nonumber
	\\ \overset{(ii)}\leq & 
	 (1+\alpha L)^N  \mathbb{E}_{\bar S_N,  i} \big\| \nabla l_i( w_{k}) \big \|   \Big \|  \prod_{j=0}^{N-1}  \big(I - \alpha \nabla^2 l_i(w_{k,j}^i)\big) -   \prod_{j=0}^{N-1}(I-\alpha \nabla^2 l_i(\widetilde w^i_{k,j})) \Big  \|   
\nonumber
	\\ &+   (1+\alpha L)^N L\mathbb{E}_{\bar S_N,  i} \big \| w_{k,N}^i  - \widetilde w^i _{k,N}  \big  \|    \nonumber
	\\\overset{(iii)}\leq & (1+\alpha L)^N  \mathbb{E}_{ i} \big\| \nabla l_i( w_{k}) \big \|   \underbrace{\mathbb{E}_{\bar S_N} \Big( \Big \|  \prod_{j=0}^{N-1}  \big(I - \alpha \nabla^2 l_i(w_{k,j}^i)\big) -   \prod_{j=0}^{N-1}(I-\alpha \nabla^2 l_i(\widetilde w^i_{k,j})) \Big  \|\, \Big |\, i \Big)}_{R(N)}    \nonumber
	\\ &+  (1+\alpha L)^N ((1+\alpha L)^N -1\big)  \frac{\sigma_g}{\sqrt{S}},  
	\end{align}
	where (i) follows from the Jensen's inequality,  (ii) follows from Lemma~\ref{le:jiw} that $\big\| \nabla l_i(\widetilde w_{k,N}^i) \big \| \leq (1+\alpha L)^N \| \nabla l_i(w_{k})\|$, and (iii) follows from item 1 in Proposition~\ref{le:distance}. 
	Our next step is to upper-bound the term $R(N)$. To simplify notations, we define a general quantity $R(m)$ by replacing $N$ in $R(N)$ with $m$,  and 
	we use $\mathbb{E}_{\bar S_m | i}(\cdot)$ to denote $\mathbb{E}_{\bar S_m}(\cdot| i)$. Then, we have 
	\begin{align}\label{eq:arjpp}
	R(m) \leq & \mathbb{E}_{\bar S_m| i}  \Big \|  \prod_{j=0}^{m-1}  \big(I - \alpha \nabla^2 l_i(w_{k,j}^i)\big) -   \prod_{j=0}^{m-2}(I-\alpha \nabla^2 l_i( w^i_{k,j})) (I-\alpha \nabla^2 l_i( \widetilde w^i_{k,m-1}) \Big  \| \nonumber
	\\ & + \mathbb{E}_{\bar S_m|i}  \Big \|  \prod_{j=0}^{m-2}(I-\alpha \nabla^2 l_i( w^i_{k,j})) (I-\alpha \nabla^2 l_i( \widetilde w^i_{k,m-1})  -  \prod_{j=0}^{m-1}(I-\alpha \nabla^2 l_i(\widetilde w^i_{k,j})) \Big  \| \nonumber
	\\\leq & (1+\alpha L)^{m-1} \alpha \rho\mathbb{E}_{\bar S_m|i} \|w_{k,m-1}^i - \widetilde w_{k,m-1}^i\| + (1+\alpha L) R(m-1) \nonumber
	\\\overset{(i)}\leq &  \alpha \rho (1+\alpha L)^{m-1} ( (1+\alpha L)^{m-1} -1 )\frac{\sigma_g}{L\sqrt{S}} + (1+\alpha L) R(m-1) \nonumber
	\\ \leq& \alpha \rho (1+\alpha L)^{N-1} \big( (1+\alpha L)^{N-1} -1  \big)\frac{\sigma_g}{L\sqrt{S}}  +  (1+\alpha L) R(m-1),
	\end{align}
	where (i) follows from Proposition~\ref{le:distance}. 
	Telescoping the above inequality over $m$ from $2$ to $N$ and using $R(1) =0$, we have 
	\begin{align}\label{addionl}
	R(N) \leq ((1+\alpha L)^{N-1}-1)^2 (1+\alpha L)^{N-1}\frac{\rho \sigma_g}{L^2\sqrt{S}}.
	\end{align}
	Thus, conditioning on $w_k$ and combining~\eqref{addionl} and~\eqref{eq:ek}, we have 
	\begin{align*}
	\|\mathbb{E} \widehat G_i(w_k)  - \nabla \mathcal{L}(w_k)\|  
	 \leq & ((1+\alpha L)^{N-1}-1)^2\frac{\rho}{L} (1+\alpha L)^{2N-1}\frac{\sigma_g}{L\sqrt{S}}  \mathbb{E}_{ i\sim p(\mathcal{T})} \big(\big\| \nabla l_i( w_{k}) \big \|   \big) \nonumber
	\\ &+   \frac{(1+\alpha L)^N ((1+\alpha L)^N -1\big)\sigma_g}{\sqrt{S}}\nonumber
	\\\leq &((1+\alpha L)^{N-1}-1)^2\frac{\rho}{L} (1+\alpha L)^{2N-1}\frac{\sigma_g}{L\sqrt{S}}  \Big( \frac{\|\nabla \mathcal{L}(w_k)\| }{1-C_l} + \frac{\sigma }{1-C_l}     \Big) \nonumber
	\\ &+   \frac{(1+\alpha L)^N ((1+\alpha L)^N -1\big)\sigma_g}{\sqrt{S}}, 
	\end{align*} 
	where the last inequality follows from Lemma \ref{le:lL}. 
	Rearranging the above inequality and using $C_{\text{\normalfont err}_1} $ and $C_{\text{\normalfont err}_2}$ defined in Proposition~\ref{th:first-est} finish the proof.  

\subsection*{Proof of Proposition~\ref{th:second}}
Recall {\small$\widehat G_i(w_k)=  \prod_{j=0}^{N-1}(I - \alpha \nabla^2 l_i(w_{k,j}^i;  D_{k,j}^i))\nabla l_i(w_{k,N}^i;  T^i_k)$}.  
Conditioning on $w_k$, we have 
	\begin{align}\label{esni}
	\mathbb{E}\|&\widehat G_i(w_k)\|^2 \nonumber
	\\\leq &\mathbb{E}_{\bar S_N, i } \bigg( \mathbb{E}_{\bar D_N, T_k^i} \Big(  \Big \|\prod_{j=0}^{N-1}(I - \alpha \nabla^2 l_i(w_{k,j}^i; D_{k,j}^i))\Big\|^2 \|\nabla l_i(w_{k,N}^i; T^i_k)\|^2 \Big | \bar S_N, i         \Big)\bigg) \nonumber
	\\\leq & \underbrace{\mathbb{E}_{\bar S_N, i } \bigg( \prod_{j=0}^{N-1} \mathbb{E}_{\bar D_N} \Big(  \Big \|I - \alpha \nabla^2 l_i(w_{k,j}^i; D_{k,j}^i)\Big\|^2 \Big | \bar S_N, i \Big)}_{P} \underbrace{\mathbb{E}_{T_k^i}\Big( \|\nabla l_i(w_{k,N}^i; T^i_k)\|^2 \Big |\bar  S_N, i         \Big)}_{Q}\bigg). 
	\end{align}
	We  next upper-bound $P$ and $Q$ in~\eqref{esni}. Note that $w_{k,j}^i, j=0,...,N-1$ are deterministic when conditioning on $S_N$, $i$, and $w_k$. Thus, conditioning on $S_N$, $i$, and $w_k$, we have 
	\begin{align}\label{pbound}
	\mathbb{E}_{\bar D_N}  \Big \|I - \alpha \nabla^2 l_i(w_{k,j}^i; D_{k,j}^i)\Big\|^2  = & \text{Var} \Big(  I - \alpha \nabla^2 l_i(w_{k,j}^i; D_{k,j}^i)  \Big) +\big\|I - \alpha \nabla^2 l_i(w_{k,j}^i) \big\|^2 \nonumber
	\\\leq & \frac{\alpha^2\sigma_H^2}{D} + (1+\alpha L)^2.
	\end{align}
	We next bound $Q$ term. Conditioning on $\bar S_N, i$ and $w_k$, we have 
	\begin{align}\label{etki}
	\mathbb{E}_{T_k^i} \|\nabla l_i(w_{k,N}^i; T^i_k)\|^2 \overset{(i)}\leq & 3\mathbb{E}_{T_k^i}\|\nabla l_i(w_{k,N}^i;  T^i_k) -\nabla l_i(w_{k,N}^i)\|^2 + 3\mathbb{E}_{T_k^i}\|\nabla l_i(w_{k,N}^i) - \nabla l_i(\widetilde w_{k,N}^i)\|^2 \nonumber
	\\ &+ 3\mathbb{E}_{T_k^i}\|\nabla l_i(\widetilde w_{k,N}^i) \|^2\nonumber
	\\\overset{(ii)}\leq & \frac{3\sigma_g^2}{T} + 3L^2 \|w_{k,N}^i - \widetilde w_{k,N}^i\|^2 + 3(1+\alpha L)^{2N} \|\nabla l_i(w_k)\|^2,
	\end{align}
	where (i) follows from the inequality that $\|\sum_{i=1}^n a\|^2\leq n\sum_{i=1}^n\|a\|^2$, and (ii) follows from Lemma~\ref{le:jiw}. Thus, conditioning on $w_k$ and combining~\eqref{esni},~\eqref{pbound} and~\eqref{etki}, we have  
	\begin{align}
	\mathbb{E}\|\widehat G_i(w_k)\|^2    \leq & 3\Big(\frac{\alpha^2\sigma_H^2}{D} + (1+\alpha L)^2\Big)^N\Big( \frac{\sigma_g^2}{T} + L^2 \mathbb{E}\|w_{k,N}^i - \widetilde w_{k,N}^i\|^2 + (1+\alpha L)^{2N} \mathbb{E}\|\nabla l_i(w_k)\|^2\Big) \nonumber
	\end{align}
	which, in conjunction with Proposition~\ref{le:distance},  
	yields 
	\begin{small}
	\begin{align}\label{ggsmida}
	\mathbb{E}\|\widehat G_i(w_k)\|^2   \leq& 3(1+\alpha L)^{2N} \Big(\frac{\alpha^2\sigma_H^2}{D} + (1+\alpha L)^2\Big)^N (\|\nabla l(w_k)\|^2 + \sigma^2)+\frac{C_{\text{\normalfont squ}_1}}{T} + \frac{C_{\text{\normalfont squ}_2}}{S}.
	\end{align} 	
	\end{small}
	\hspace{-0.12cm}Based on Lemma~\ref{le:lL} and  conditioning on $w_k$, we have  
	\begin{align*}
	\|\nabla l(w_k)\|^2 \leq \frac{2}{(1-C_l)^2} \|\nabla \mathcal{L}(w_k)\| + \frac{2C_l^2}{(1-C_l)^2} \sigma^2,
	\end{align*}
	which, in conjunction with  $\frac{2x^2}{(1-x)^2}+1 \leq \frac{2}{(1-x)^2}$ and \eqref{ggsmida}, finishes the proof. 



\subsection*{Proof of Theorem~\ref{th:mainonline}}
 The proof of Theorem~\ref{th:mainonline} consists of four main steps: step $1$ of bounding an iterative meta update by the meta-gradient smoothness established by Proposition~\ref{th:lipshiz}; step $2$ of characterizing first-moment  error of the meta-gradient estimator { $\widehat G_i(w_k)$} by Proposition~\ref{th:first-est}; step $3$ of characterizing second-moment  error of the meta-gradient estimator { $\widehat G_i(w_k)$} by Proposition~\ref{th:second}; and step $4$ of combining steps 1-3, and telescoping to yield the convergence. 
 
To simplify notations, define the smoothness parameter of the meta-gradient as $$L_{w_k} = (1+\alpha L)^{2N}L + C_\mathcal{L} \mathbb{E}_{i\sim p(\mathcal{T})}\|\nabla l_i(w_k)\|,$$ where $C_\mathcal{L}$ is given in~\eqref{clcl}. 	
Based on the smoothness of the gradient $\nabla \mathcal{L}(w) $ given by Proposition~\ref{th:lipshiz}, we have 
	\begin{align*}
	\mathcal{L}(w_{k+1}) \leq & \mathcal{L}(w_k) + \langle  \nabla \mathcal{L}(w)  , w_{k+1} - w_{k}    \rangle + \frac{L_{w_k}}{2} \|w_{k+1}-w_{k}\|^2 \nonumber
	\end{align*}
	Note that the randomness from $\beta_k$ depends on $B_k^\prime$ and $D_{L_k}^i, i \in B_k^{\prime}$, and thus is independent of $S_{k,j}^i, D_{k,j}^i$ and $T_k^i$ for $i\in B_k, j=0,...,N$. Then,  taking expectation over the above inequality, conditioning on $w_k$, and recalling $e_k := \mathbb{E}\widehat G_i(w_k) - \nabla \mathcal{L}(w_k) $, we have 
	\begin{align*}
	\mathbb{E}( \mathcal{L}(w_{k+1})| w_k) \leq \mathcal{L}(w_{k}) - \mathbb{E} &(\beta_k)\langle \nabla \mathcal{L}(w_{k}), \nabla \mathcal{L}(w_{k}) + e_k\rangle+\frac{L_{w_k}\mathbb{E} (\beta^2_k)\mathbb{E} \big\|  \frac{1}{B} \sum_{i\in B_k} \widehat G_i(w_k) \big \|^2}{2} . 
	\end{align*}
Then, applying Lemma~\ref{le:xiaodege} in the above inequality yields
	\begin{align}\label{qiangxing}
	\mathbb{E}( \mathcal{L}(w_{k+1})| w_k)
	\leq & \mathcal{L}(w_{k}) -\frac{4}{5C_\beta} \frac{1}{ L_{w_k}}  \|\nabla \mathcal{L}(w_{k})\|^2- \frac{4}{5C_\beta} \frac{1}{ L_{w_k}} \langle \nabla \mathcal{L}(w_{k}), e_k\rangle  \nonumber
	\\ &+ \frac{2}{C_\beta^2} \frac{1}{L_{w_k}} \Big( \frac{1}{B}\mathbb{E}  \big\|   \widehat G_i(w_k) \big \|^2  +    \|\mathbb{E}  \widehat G_i(w_k)\|^2\Big).\nonumber
	\\\leq & \mathcal{L}(w_{k}) -\frac{4}{5C_\beta} \frac{1}{ L_{w_k}}  \|\nabla \mathcal{L}(w_{k})\|^2+ \frac{2}{5C_\beta} \frac{1}{ L_{w_k}} \| \nabla \mathcal{L}(w_{k})\|^2 +  \frac{2}{5C_\beta} \frac{1}{ L_{w_k}}\| e_k\|^2  \nonumber
	\\ &+ \frac{2}{C_\beta^2} \frac{1}{L_{w_k}} \Big( \frac{1}{B}\mathbb{E}  \big\|   \widehat G_i(w_k) \big \|^2  +    \|\mathbb{E}  \widehat G_i(w_k)\|^2\Big).
	\end{align}
	Then, 
	applying Propositions~\ref{th:first-est} and~\ref{th:second} to  the above inequality yields
	\begin{align}\label{sikas}
	\mathbb{E}(& \mathcal{L}(w_{k+1})| w_k)  \nonumber
	 \\\leq & \mathcal{L}(w_{k}) -\frac{2}{5C_\beta} \frac{1}{ L_{w_k}}  \|\nabla \mathcal{L}(w_{k})\|^2+ \frac{2}{C_\beta^2} \frac{1}{L_{w_k}}  \frac{1}{B}\mathbb{E}  \big\|   \widehat G_i(w_k) \big \|^2  +    \frac{4}{C_\beta^2} \frac{1}{L_{w_k}} \|\nabla \mathcal{L}(w_{k})\|^2    \nonumber
	\\ &+ \Big(\frac{6}{5C_\beta L_{w_k}}  + \frac{12}{C_\beta^2 L_{w_k}} \Big) \Big ( \frac{C^2_{{\text{\normalfont err}}_2}  }{S}\|\nabla \mathcal{L}(w_k)\|^2 +\frac{C^2_{{\text{\normalfont err}}_1}}{S} + \frac{C^2_{{\text{\normalfont err}}_2} \sigma^2  }{S}\Big ) 
	  \nonumber
	\\ \leq & \mathcal{L}(w_{k}) - \frac{2}{C_\beta L_{w_k}} \left(   \frac{1}{5} - \left(\frac{3}{5} + \frac{6}{C_\beta}\right)\frac{C^2_{{\text{\normalfont err}}_2} }{S} - \frac{C_{\text{\normalfont squ}_3}}{C_\beta B} - \frac{2}{C_\beta}\right) \|\nabla \mathcal{L}(w_k) \|^2 \nonumber
	\\ &+ \frac{6}{C_\beta L_{w_k}S}\Big( \frac{1}{5} +\frac{2}{C_\beta}   \Big)\Big( C^2_{{\text{\normalfont err}}_1} +C^2_{{\text{\normalfont err}}_2}\sigma^2\Big)  +  \frac{2}{C_\beta^2 L_{w_k}B} \Big(  \frac{C_{\text{\normalfont squ}_1}}{T}  +  \frac{C_{\text{\normalfont squ}_2}}{S} + C_{\text{\normalfont squ}_3} \sigma^2 \Big).
	\end{align}
	Recalling {$L_{w_k} = (1+\alpha L)^{2N}L + C_\mathcal{L} \mathbb{E}_{i}\|\nabla l_i(w_k)\|$}, we have  $ L_{w_k} \geq L$ and 
	\begin{align}\label{oips}
	L_{w_k} 
	 \overset{(i)}\leq &  (1+\alpha L)^{2N}L + \frac{C_\mathcal{L}\sigma}{1-C_l}+ \frac{C_\mathcal{L}}{1-C_l} \|\nabla \mathcal{L}(w_k)\|,  
	\end{align}
	where (i) follows from Assumption~\ref{a2} and Lemma~\ref{le:lL}. 
	Combining~\eqref{sikas} and~\eqref{oips} yields 
	\begin{align}\label{miops}
	\mathbb{E}( \mathcal{L}(w_{k+1})| w_k) \leq& \mathcal{L}(w_{k}) + \frac{6}{C_\beta L}\Big( \frac{1}{5} +\frac{2}{C_\beta}   \Big)\Big( C^2_{{\text{\normalfont err}}_1} +C^2_{{\text{\normalfont err}}_2}\sigma^2\Big) \frac{1}{S} \nonumber
	\\&+  \frac{2}{C_\beta^2 L} \Big(  \frac{C_{\text{\normalfont squ}_1}}{T}  +  \frac{C_{\text{\normalfont squ}_2}}{S} + C_{\text{\normalfont squ}_3} \sigma^2 \Big) \frac{1}{B} \nonumber
	\\ &	- \frac{2}{C_\beta } \frac{ \frac{1}{5} - \left(\frac{3}{5} + \frac{6}{C_\beta}\right)\frac{C^2_{{\text{\normalfont err}}_2} }{S} - \frac{C_{\text{\normalfont squ}_3}}{C_\beta B} - \frac{2}{C_\beta}}{(1+\alpha L)^{2N}L + \frac{C_\mathcal{L}\sigma}{1-C_l}+ \frac{C_\mathcal{L}}{1-C_l} \|\nabla \mathcal{L}(w_k)\|} \|\nabla \mathcal{L}(w_k) \|^2.
	\end{align}
	Based on the notations in \eqref{para:com}, we rewrite~\eqref{miops} as  
	\begin{align*}
	\mathbb{E}&( \mathcal{L}(w_{k+1})| w_k) \leq \mathcal{L}(w_{k}) + \frac{\xi}{S} +  \frac{\phi }{B}	-\theta \frac{  \|\nabla \mathcal{L}(w_k) \|^2}{\chi +\|\nabla \mathcal{L}(w_k) \|}.
	\end{align*}
	Unconditioning on $w_k$ in the above inequality and  
	Telescoping the above inequality over $k$ from $0$ to $K-1$, we have 
	\begin{align}\label{ggopos}
	\frac{1}{K}\sum_{k=0}^{K-1} \mathbb{E}\left(\frac{ \theta \|\nabla \mathcal{L}(w_k) \|^2}{\chi +\|\nabla \mathcal{L}(w_k) \|}\right) \leq \frac{\Delta}{K} +    \frac{\xi}{S} +  \frac{\phi }{B},
	\end{align}
	where $\Delta = \mathcal{L}(w_0) - \mathcal{L}^*$. 
	Choosing $\zeta$  from $\{0,...,K-1\}$ uniformly at random, we obtain from \eqref{ggopos} that 
	\begin{align}\label{medistep}
	\mathbb{E}\left(\frac{ \theta \|\nabla \mathcal{L}(w_\zeta) \|^2}{\chi +\|\nabla \mathcal{L}(w_\zeta) \|}\right) \leq \frac{\Delta}{K} +    \frac{\xi}{S} +  \frac{\phi }{B}.
	\end{align}
	Consider a function $f(x) = \frac{x^2}{c+x}, \,x>0$, where $c>0$ is a constant. Simple computation shows that $f^{\prime\prime}(x) =\frac{2c^2}{(x+c)^3}>0$. Thus, using Jensen's inequality in \eqref{medistep}, we have 
	\begin{align}\label{reoolls}
	\frac{ \theta (\mathbb{E}\|\nabla \mathcal{L}(w_\zeta) \|)^2}{\chi +\mathbb{E}\|\nabla \mathcal{L}(w_\zeta) \|} \leq \frac{\Delta}{K} +    \frac{\xi}{S} +  \frac{\phi }{B}.
	\end{align}
	Rearranging the above inequality yields
	\begin{align}\label{havetogo}
	\mathbb{E}\|\nabla \mathcal{L}(w_\zeta) \|  \leq &\frac{\Delta}{2\theta }\frac{1}{K} +    \frac{\xi}{2\theta}\frac{1}{S} +  \frac{\phi }{2\theta}\frac{1}{B} + \sqrt{ \chi \Big(\frac{\Delta}{2\theta }\frac{1}{K} +    \frac{\xi}{2\theta}\frac{1}{S} +  \frac{\phi }{2\theta}\frac{1}{B}\Big)   + \Big(\frac{\Delta}{2\theta }\frac{1}{K} +    \frac{\xi}{2\theta}\frac{1}{S} +  \frac{\phi }{2\theta}\frac{1}{B}\Big)^2 } \nonumber
	\\ \leq &\frac{\Delta}{\theta }\frac{1}{K} +    \frac{\xi}{\theta}\frac{1}{S} +  \frac{\phi }{\theta}\frac{1}{B} + \sqrt{\frac{\chi}{2} }\sqrt{\frac{\Delta}{\theta }\frac{1}{K} +    \frac{\xi}{\theta}\frac{1}{S} +  \frac{\phi }{\theta}\frac{1}{B}},
	\end{align}
	which finishes the proof.
\subsection*{Proof of Corollary~\ref{co:online}}
	Since $\alpha = \frac{1}{8NL}$, we have 
	\begin{align*}
	(1+\alpha L)^ N = &\big(1+ \frac{1}{8N}\big)^N = e^{N\log(1+\frac{1}{8N})} \leq e^{1/8}  < \frac{5}{4},
	(1+\alpha L)^ {2N} < e^{1/4} <  \frac{3}{2},
	\end{align*}
	which, in conjunction with~\eqref{ppolp}, implies that 
	\begin{align}\label{errpp}
	C_{{\text{\normalfont err}}_1}  < \frac{5\sigma_g}{16},\quad C_{{\text{\normalfont err}}_2} < \frac{3\rho \sigma_g}{4L^2}.
	\end{align}
	Furthermore, noting that $D \geq \sigma_H^2/L^2$, we have
	\begin{align}\label{ctextp}
	C_{\text{\normalfont squ}_1} \leq &3(1+2\alpha L + 2\alpha^2 L^2)^N\sigma_g^2 < 3 e^{9/32}\sigma_g^2<4\sigma_g^2, \; C_{\text{\normalfont squ}_2} < \frac{1.3\sigma^{2}_g}{8} < \frac{\sigma_g^2}{5},\; C_{\text{\normalfont squ}_3} \leq 11.
	\end{align}
	Based on~\eqref{clcl}, we have 
	\begin{align}\label{bb:cl}
	C_{\mathcal{L}}<& \frac{75}{128}\frac{\rho}{L}<\frac{3}{5}\frac{\rho}{L} \,\text{ and } \,C_{\mathcal{L}} \overset{(i)}>  \frac{\rho}{L} ((N-1) \alpha L) > \frac{1}{16}\frac{\rho}{L},
	\end{align}
	where (i) follows from the inequality that $(1+a)^n > 1 +an$. 
	Then, using \eqref{errpp},~\eqref{ctextp} and~\eqref{bb:cl}, we obtain from \eqref{para:com} that 
	\begin{align}\label{manypara}
	\xi < &\frac{7}{500L} \Big( \frac{1}{10} + \frac{9\rho\sigma^2}{16L^4}\Big)\sigma_g^2,\quad \phi \leq \frac{1}{5000L} \Big( \frac{3\sigma_g^2}{T} + \frac{\sigma_g^2}{5S} + 11\sigma^2\Big) < \frac{1}{1000L}(\sigma_g^2 + 3\sigma^2)  \nonumber
	\\ \theta \geq & \frac{L}{60\rho} \Big( \frac{1}{5} - \frac{4}{5} \frac{9}{16} \frac{\rho^2\sigma_g^2}{L^4}\frac{1}{S}  - \frac{11}{100B} - \frac{1}{50}\Big) 
	 = \frac{L}{1500\rho},\; \chi \leq   \frac{24L^2}{\rho} +\sigma.
	\end{align}
	Then, treating $\Delta, \rho, L$ as constants and using~\eqref{c:result}, we obtain
	\begin{small}
	\begin{align*}
	\mathbb{E}\|\nabla \mathcal{L}(w_\zeta) \|  \leq \mathcal{O} \Big(  \frac{1}{K} + \frac{\sigma_g^2(\sigma^2+1)}{S} + \frac{\sigma_g^2 +\sigma^2}{B} +\frac{\sigma^2_g}{TB}+ \sqrt{\sigma +1}\sqrt{\frac{1}{K} + \frac{\sigma_g^2(\sigma^2+1)}{S} + \frac{\sigma_g^2 +\sigma^2}{B}+\frac{\sigma^2_g}{TB}}\Big).
	\end{align*}
	\end{small}
	\hspace{-0.15cm}Then, choosing batch sizes $S\geq C_S\sigma_g^2(\sigma^2+1)\max(\sigma,1)\epsilon^{-2}$, $B\geq C_B(\sigma_g^2+\sigma^2)\max(\sigma,1)\epsilon^{-2}$ and $TB >C_{T}\sigma_g^2\max(\sigma,1) \epsilon^{-2}$, we have 
	\begin{align*}
	\mathbb{E}\|\nabla \mathcal{L}(w_\zeta) \|  \leq  \mathcal{O}\bigg(\frac{1}{K} + \frac{1}{\epsilon^2} \Big(\frac{1}{C_S} +\frac{1}{C_B}+\frac{1}{C_{T}} \Big)+ \sqrt{\sigma} \sqrt{\frac{1}{K}+\frac{1}{\sigma\epsilon^2}\Big(\frac{1}{C_S} +\frac{1}{C_B}+\frac{1}{C_{T}} \Big)}\bigg)
	\end{align*}
	After at most $K =  C_K\max(\sigma,1)\epsilon^{-2}$ iterations, the above inequality implies, for constants $C_S, C_B,C_T$ and $C_K$ large enough, 
	$\mathbb{E}\|\nabla \mathcal{L}(w_\zeta)\| \leq \epsilon$.
	Recall that we need $|B_k^\prime| > \frac{4C^2_{\mathcal{L}}\sigma^2}{3(1+\alpha L)^{4N}L^2}$ and $|D_{L_k}^i| > \frac{64\sigma^2_g C_\mathcal{L}^2}{(1+\alpha L)^{4N}L^2}$ for building stepsize $\beta_k$ at each iteration $k$. Based on the selected parameters, we have 
	\begin{align*}
	\frac{4C^2_{\mathcal{L}}\sigma^2}{3(1+\alpha L)^{4N}L^2} \leq \frac{4\sigma^2}{3L^2} \frac{3\rho}{5L}\leq \Theta({\sigma^2}), \quad \frac{64\sigma^2_g C_\mathcal{L}^2}{(1+\alpha L)^{4N}L^2} < \Theta(\sigma_g^2),
	\end{align*}
	which implies  $|B_k^\prime| =\Theta(\sigma^2)$ and $|D_{L_k}^i| =\Theta(\sigma^2_g)$. Then, since  the batch size $D =\Theta(\sigma_H^2/L^2)$, the total number of gradient computations at each meta iteration $k$ is given by 
	$B (NS+T) + |B_k^\prime||D_{L_k}^i|\leq \mathcal{O}(N\epsilon^{-4}+\epsilon^{-2}    )$.
	Furthermore, the total number of Hessian computations at each meta iteration is given by 
	$BND \leq \mathcal{O}(N\epsilon^{-2}). $
	This completes the proof. 

\subsection{Proofs for Section~\ref{theory:offline}: Convergence of Multi-Step MAML in Finite-Sum Case}\label{proof:meta_finite}
In this subsection, we provide proofs for the convergence properties of multi-step MAML in the finite-sum case.
\subsection*{Proof of Proposition~\ref{finite:lip}}
By the definition of $\nabla \mathcal{L}_i(\cdot)$, we have 
	\begin{align}\label{lopasv}
	\|\nabla \mathcal{L}_i(w)  -\nabla \mathcal{L}_i(u) \| 
	\leq &\Big\|\prod_{j=0}^{N-1}(I - \alpha \nabla^2 l_{S_i}(\widetilde w_{j}^i))\nabla l_{T_i}(\widetilde w_{N}^i) -\prod_{j=0}^{N-1}(I - \alpha \nabla^2 l_{S_i}(\widetilde u_{j}^i))\nabla l_{T_i}(\widetilde w_{N}^i)\Big\|  \nonumber
	\\ & + \Big\|\prod_{j=0}^{N-1}(I - \alpha \nabla^2 l_{S_i}(\widetilde u_{j}^i))\nabla l_{T_i}(\widetilde w_{N}^i) -\prod_{j=0}^{N-1}(I - \alpha \nabla^2 l_{S_i}(\widetilde u_{j}^i))\nabla l_{T_i}(\widetilde u_{N}^i)\Big\| \nonumber
	\\ \leq &\underbrace{ \Big\|\prod_{j=0}^{N-1}(I - \alpha \nabla^2 l_{S_i}(\widetilde  w_{j}^i)) -\prod_{j=0}^{N-1}(I - \alpha \nabla^2 l_{S_i}(\widetilde u_{j}^i))\Big\|}_{A}  \|\nabla l_{T_i}(\widetilde  w_{N}^i)\|  \nonumber
	\\& + (1+\alpha L)^N \|\nabla l_{T_i}(\widetilde w_{N}^i)- \nabla l_{T_i}(\widetilde  u_{N}^i)\|.
	\end{align}
	We next upper-bound $A$ in the above inequality. Specifically,  we have
	\begin{align}\label{alegeq}
	A \leq &  \Big\|\prod_{j=0}^{N-1}(I - \alpha \nabla^2 l_{S_i}(\widetilde w_{j}^i)) -\prod_{j=0}^{N-2}(I - \alpha \nabla^2 l_{S_i}(\widetilde w_{j}^i))(I - \alpha \nabla^2 l_{S_i}(\widetilde u_{N-1}^i))\Big\| \nonumber
	\\ &+\Big\| \prod_{j=0}^{N-2}(I - \alpha \nabla^2 l_{S_i}(\widetilde w_{j}^i))(I - \alpha \nabla^2 l_{S_i}(\widetilde u_{N-1}^i))-\prod_{j=0}^{N-1}(I - \alpha \nabla^2 l_{S_i}(\widetilde u_{j}^i))\Big\|\nonumber
	\\ \leq &\Big(   (1+\alpha  L)^{N-1}\alpha \rho  + \frac{\rho}{L} (1+\alpha L)^N \big( (1+\alpha L)^{N-1} -1 \big)\Big)\|w-u\|,
	\end{align}
	where the last inequality uses an approach similar to \eqref{vbbn}. 
	Combining~\eqref{lopasv} and \eqref{alegeq} yields
	\begin{align}\label{inops}
	\|\nabla \mathcal{L}_i(w) & -\nabla \mathcal{L}_i(u) \| \nonumber
\\	 \leq& \big(   (1+\alpha  L)^{N-1}\alpha \rho  + \frac{\rho}{L} (1+\alpha L)^N \big( (1+\alpha L)^{N-1} -1 \big)\big)\|w-u\|  \|\nabla l_{T_i}(\widetilde w_{N}^i)\| \nonumber
	\\ &+ (1+\alpha L)^NL \|\widetilde w_{N}^i- \widetilde u_{N}^i\|.
	\end{align}
	To upper-bound $ \|\nabla l_{T_i}(\widetilde w_{N}^i)\| $ in~\eqref{inops},  using the mean value theorem, we have
	\begin{align}\label{lowni}
	\|\nabla l_{T_i}(\widetilde w_{N}^i)\|  = &  \Big\|\nabla l_{T_i} (w-\sum_{j=0}^{N-1}\alpha \nabla l_{S_i}(\widetilde w_j^i))\Big\|\nonumber
	\\ \overset{(i)}\leq & \|\nabla l_{T_i} (w)\| + \alpha L \sum_{j=0}^{N-1} (1+\alpha L)^j\big\| \nabla l_{S_i}(w)   \big\| \nonumber
	\\ \overset{(ii)}\leq & (1+\alpha L)^N  \|\nabla l_{T_i} (w)\|  + \big( (1+\alpha L)^N-1 \big)b_i,
	\end{align}
	where (i) follows from Lemma~\ref{finite:gbd}, and (ii) follows from Assumption~\ref{assum:vaoff}. In addition, using an approach similar to Lemma~\ref{d_u_w}, we have
	\begin{align}\label{wnos}
	\|\widetilde w_{N}^i- \widetilde u_{N}^i\| \leq (1+\alpha L)^N \|w-u\|.
	\end{align}
	Combining~\eqref{inops}, \eqref{lowni} and \eqref{wnos} yields
	\begin{align*}
	\|\nabla \mathcal{L}_i(w) & -\nabla \mathcal{L}_i(u) \| 
	\\ \leq& \Big(   (1+\alpha  L)^{N-1}\alpha \rho  + \frac{\rho}{L} (1+\alpha L)^N \big( (1+\alpha L)^{N-1} -1 \big)\Big)(1+\alpha L)^N \|\nabla l_{T_i} (w)\|\|w-u\| \nonumber
	\\&+ \Big(   (1+\alpha  L)^{N-1}\alpha \rho  + \frac{\rho}{L} (1+\alpha L)^N \big( (1+\alpha L)^{N-1} -1 \big)\Big)\big( (1+\alpha L)^N-1 \big)b_i\|w-u\| \nonumber
	\\ &+ (1+\alpha L)^{2N}L \|w- u\|,
	\end{align*}
	which, in conjunction with $C_b$ and $C_\mathcal{L}$ given in \eqref{cl1ss}, yields 
	\begin{align*}
	\|\nabla \mathcal{L}_i(w)  -\nabla \mathcal{L}_i(u) \| \leq \big((1+\alpha L)^{2N}L + C_bb_i + C_{\mathcal{L}} \|\nabla l_{T_i}(w)\|  \big)\|w-u\|.
	\end{align*}
	Based on the above inequality and  Jensen's inequality, we 
	finish the proof.

\subsection*{Proof of Proposition~\ref{finite:seconderr}}
	Conditioning on $w_k$, we have 
	\begin{align*}
	\mathbb{E}\|\widehat G_i(w_k)\|^2 = &\mathbb{E} \Big\| \prod_{j=0}^{N-1}(I - \alpha \nabla^2 l_{S_i}(w_{k,j}^i))\nabla l_{T_i}(w_{k,N}^i)  \Big\|^2 \leq  (1+\alpha L)^{2N} \mathbb{E} \|\nabla l_{T_i}(w_{k,N}^i)\|^2,
	\end{align*}
	which, using an approach similar to \eqref{lowni}, yields
	\begin{align}\label{gwkopo}
	\mathbb{E}\|\widehat G_i(w_k)\|^2 \leq&  (1+\alpha L)^{2N} 2(1+\alpha L)^{2N} \mathbb{E} \|\nabla l_{T_i}(w_k)\|^2 + 2(1+\alpha L)^{2N} \big( (1+\alpha L)^N -1\big)^2 \mathbb{E}_i b_i^2 \nonumber
	\\ \leq  & 2(1+\alpha L)^{4N} (\|\nabla l_{T}(w_k)\|^2 + \sigma^2)+ 2(1+\alpha L)^{2N} \big( (1+\alpha L)^N -1\big)^2 \widetilde b \nonumber
	\\ \overset{(i)}\leq & 2(1+\alpha L)^{4N} \Big(  \frac{2}{C_1^2} \|\nabla l_{T}(w_k)\|^2 + \frac{2C_2^2}{C_1^2} + \sigma^2 \Big) + 2(1+\alpha L)^{2N} \big( (1+\alpha L)^N -1\big)^2 \widetilde b \nonumber
	\\ \leq & \frac{4(1+\alpha L)^{4N}}{C_1^2}\|\nabla l_{T}(w_k)\|^2 + \frac{4(1+\alpha L)^{4N}C_2^2}{C_1^2} + 2(1+\alpha L)^{4N}(\sigma^2 + \widetilde b),
	\end{align}
	where (i) follows from Lemma~\ref{twc1c2}, and constants $C_1$ and $C_2$ are given by \eqref{c1c2}. Noting that $C_2=\big( (1+\alpha L)^{2N}-1  \big)\sigma + (1+\alpha L)^N \big((1+\alpha L)^N -1 \big) b < \big( (1+\alpha L)^{2N}-1  \big)(\sigma +b)$ and using the definitions of $A_{\text{\normalfont squ}_1}, A_{\text{\normalfont squ}_2}$ in \eqref{wocaopp}, we finish the proof. 
\subsection*{Proof of Theorem~\ref{mainth:offline}}
	Based on the smoothness of $\nabla \mathcal{L}(\cdot)$ established in Proposition~\ref{finite:lip}, we have
	\begin{align*}
	\mathcal{L}(w_{k+1}) 
	 \leq &\mathcal{L}(w_k) -\beta_k\Big \langle \nabla \mathcal{L}(w_k), \frac{1}{B}\sum_{i\in B_k} \widehat G_i(w_k)\Big\rangle + \frac{L_{w_k}\beta_k^2}{2}\Big\|\frac{1}{B}\sum_{i\in B_k} \widehat G_i(w_k)\Big\|^2 \nonumber
	\end{align*}
	Taking the conditional expectation given $w_k$ over the above inequality and noting that the randomness over $\beta_k$ is independent of the randomness over $ \widehat G_i(w_k)$,  we have 
	\begin{small}
	\begin{align}\label{wk1k1}
	\mathbb{E}	(\mathcal{L}&(w_{k+1})  | w_k) \nonumber
	\\ \leq &\mathcal{L}(w_{k}) - \frac{1}{C_\beta}\mathbb{E}\Big(\frac{1}{\hat L_{w_k}} \,\Big |\, w_k\Big) \|\nabla \mathcal{L}(w_{k}) \|^2+  \frac{L_{w_k}}{2C_\beta^2}\mathbb{E}\Big(\frac{1}{\hat L^2_{w_k}} \,\Big |\, w_k\Big) \mathbb{E} \Big(  \Big\|\frac{1}{B}\sum_{i\in B_k} \widehat G_i(w_k)\Big\|^2  \Big| w_k   \Big).
	\end{align}
	\end{small}
\hspace{-0.14cm}	Note that, conditioning on $w_k$,
	\begin{align}\label{bbe1b}
	\mathbb{E}  \Big\|\frac{1}{B}\sum_{i\in B_k} \widehat G_i(w_k)\Big\|^2   
	\leq & \frac{1}{B}\big(  A_{\text{\normalfont squ}_1} \|\nabla \mathcal{L}(w_k)\|^2 + A_{\text{\normalfont squ}_2}   \big)  + \|\nabla \mathcal{L}(w_k)\|^2
	\end{align}
	where the inequality follows from Proposition~\ref{finite:seconderr}. Then, combining~\eqref{bbe1b},~\eqref{wk1k1} and applying Lemma~\ref{le:betak}, we have 
	\begin{align}\label{lwkpkps}
	\mathbb{E}	(\mathcal{L}(w_{k+1})  | w_k)
	 \leq & \mathcal{L}(w_{k})  - \Big(  \frac{1}{L_{w_k}C_\beta}  - \frac{1}{L_{w_k}C_\beta^2}\Big( \frac{A_{\text{\normalfont squ}_1}}{B}+1\Big)\Big)\|\nabla \mathcal{L}(w_k)\|^2 + \frac{A_{\text{\normalfont squ}_2}}{L_{w_k}C_\beta^2b}.
	\end{align}
	Recalling that $L_{w_k} = (1+\alpha L)^{2N}L + C_b b +  C_\mathcal{L} \mathbb{E}_{i\sim p(\mathcal{T})}\|\nabla l_{T_i}(w_k)\|$ and conditioning on $w_k$, we have  $L_{w_k}\geq L$ and 
	\begin{align}\label{lkulpi}
	L_{w_k} \leq & (1+\alpha L)^{2N}L + C_b b +  C_\mathcal{L} (\|\nabla l_{T}(w_k)\| + \sigma) \nonumber
	\\\overset{(i)}\leq&(1+\alpha L)^{2N}L + C_b b + C_\mathcal{L}\Big( \frac{C_2}{C_1} +\sigma\Big)  + \frac{C_\mathcal{L}}{C_1} \|\nabla \mathcal{L}(w_k)\|,
	\end{align}
	where $(i)$ follows from Lemma~\ref{twc1c2}. Combining~\eqref{lkulpi} and~\eqref{lwkpkps} yields
	\begin{small}
	\begin{align}\label{b2n8}
	\mathbb{E}	(&\mathcal{L}(w_{k+1})  | w_k) \nonumber
	\\\leq&  \mathcal{L}(w_{k})  -\frac{\Big(  \frac{1}{C_\beta}  - \frac{1}{C_\beta^2}\Big( \frac{A_{\text{\normalfont squ}_1}}{B}+1\Big)\Big)\|\nabla \mathcal{L}(w_k)\|^2 }{(1+\alpha L)^{2N}L + C_b b + C_\mathcal{L}\Big( \frac{C_2}{C_1} +\sigma\Big)  + \frac{C_\mathcal{L}}{C_1} \|\nabla \mathcal{L}(w_k)\|} + \frac{1}{LC_\beta^2} \frac{A_{\text{\normalfont squ}_2}}{B} \nonumber
	\\ = &\mathcal{L}(w_{k})  -\frac{\frac{C_1}{C_\mathcal{L}} \Big(  \frac{1}{C_\beta}  - \frac{1}{C_\beta^2}\Big( \frac{A_{\text{\normalfont squ}_1}}{B}+1\Big)\Big)\|\nabla \mathcal{L}(w_k)\|^2 }{\frac{C_1}{C_\mathcal{L}} (1+\alpha L)^{2N}L + \frac{bC_1C_b  }{C_\mathcal{L}} +C_2 +C_1\sigma  + \|\nabla \mathcal{L}(w_k)\|} + \frac{1}{LC_\beta^2} \frac{A_{\text{\normalfont squ}_2}}{B} \nonumber
	\\= &\mathcal{L}(w_{k})  -\frac{\frac{C_1}{C_\mathcal{L}} \Big(  \frac{1}{C_\beta}  - \frac{1}{C_\beta^2}\Big( \frac{A_{\text{\normalfont squ}_1}}{B}+1\Big)\Big)\|\nabla \mathcal{L}(w_k)\|^2 }{\frac{C_1}{C_\mathcal{L}} (1+\alpha L)^{2N}L + \frac{bC_1C_b  }{C_\mathcal{L}} +(1+\alpha L)^N((1+\alpha L)^{2N}-1)b  + \|\nabla \mathcal{L}(w_k)\|} +  \frac{A_{\text{\normalfont squ}_2}}{ LC_\beta^2 B},
	\end{align}
	\end{small}
	\hspace{-0.12cm}where the last equality follows from the definitions of $C_1,C_2$ in \eqref{c1c2}. 
Combining the definitions in~\eqref{offline:constants}  with \eqref{b2n8} and taking the expectation over  $w_k$,  we have
	\begin{align*}
	\mathbb{E}\frac{\theta \|\nabla \mathcal{L}(w_k)\|^2}{\xi + \|\nabla \mathcal{L}(w_k)\|} \leq \mathbb{E}( \mathcal{L}(w_{k})  - \mathcal{L}(w_{k+1})     ) + \frac{\phi}{B}.
	\end{align*}
	Telescoping the above bound over $k$ from $0$ to $K-1$ and choosing $\zeta$  from $\{0,...,K-1\}$ uniformly at random, we have
	\begin{align}\label{oppps}
	\mathbb{E}\frac{\theta \|\nabla \mathcal{L}(w_\zeta)\|^2}{\xi + \|\nabla \mathcal{L}(w_\zeta)\|}  \leq \frac{\Delta}{K} +\frac{\phi}{B}. 
	\end{align}
	Using an approach similar to \eqref{reoolls}, we obtain from~\eqref{oppps} that 
	\begin{align*}
	\frac{	(\mathbb{E}\|\nabla \mathcal{L}(w_\zeta)\|)^2}{\xi + 	\mathbb{E}\|\nabla \mathcal{L}(w_\zeta)\|}  \leq \frac{\Delta}{\theta K} +\frac{\phi}{\theta B},
	\end{align*}
	which further implies that 
	\begin{align}\label{iolscasa}
	\mathbb{E}\|\nabla \mathcal{L}(w_\zeta)\| \leq \frac{\Delta}{2\theta K} +\frac{\phi}{2\theta B} + \sqrt{ \xi \Big(\frac{\Delta}{\theta K} +\frac{\phi}{\theta B}\Big) + \Big(\frac{\Delta}{2\theta K} +\frac{\phi}{2\theta B}\Big)^2 },
	\end{align}
	which finishes the proof. 
\subsection*{Proof of Corollary~\ref{co:mainoffline}}
	Since $\alpha = \frac{1}{8NL}$, we have $(1+\alpha L)^{4N}< e^{0.5}<2$, and thus
	\begin{align}\label{afterpuck}
	A_{\text{\normalfont squ}_1}  &<  32, \; A_{\text{\normalfont squ}_2} < 8(\sigma +b)^2 + 4(\sigma^2+\widetilde b),\nonumber
	\\C_{\mathcal{L}} &< \Big(\frac{5\rho}{32NL} + \frac{\rho}{L}\frac{5}{16}\Big) \frac{5}{4} < \frac{5\rho}{8L}, \; C_{\mathcal{L}} > \frac{\rho}{L} \big( (1+\alpha L)^{N-1}-1\big) > \frac{\rho}{L} \alpha L (N-1)>\frac{\rho}{16L},\nonumber
	\\C_b &< \frac{15}{32}\frac{\rho}{L}\frac{1}{4}< \frac{\rho}{8L},
	\end{align}
	which, in conjunction with \eqref{offline:constants}, yields
	\begin{align}\label{fini:offpara}
	\theta \geq &  \frac{1}{80} \frac{4L}{5\rho} \Big( 1- \frac{33}{80}\Big) \geq \frac{L}{200\rho}, \; \phi \leq \frac{2(\sigma +b)^2 + (\sigma^2+\widetilde b)}{1600L},\; \xi \leq  \frac{24L^2}{\rho} + \frac{37b}{16}. 
	\end{align}
	Combining~\eqref{fini:offpara} and \eqref{iopnn} yields
	\begin{align*}
	\mathbb{E}\|\nabla \mathcal{L}(w_\zeta)\| \leq &\frac{\Delta}{2\theta K} +\frac{\phi}{2\theta B} + \sqrt{ \xi \Big(\frac{\Delta}{\theta K} +\frac{\phi}{\theta B}\Big) + \Big(\frac{\Delta}{2\theta K} +\frac{\phi}{2\theta B}\Big)^2 } \nonumber
	\\ \leq & \mathcal{O}\Big(  \frac{1}{K} +\frac{\sigma^2}{B} +\sqrt{\frac{1}{K} +\frac{\sigma^2}{B} }  \Big).
	\end{align*}
	Then, based on the parameter selection that $B\geq C_B\sigma^2\epsilon^{-2}$ and after at most $K=C_k\epsilon^{-2}$ iterations, we have 
	\begin{align*}
	\mathbb{E}\|\nabla \mathcal{L}(w_\zeta)\| \leq \mathcal{O}\Big(\big(\frac{1}{C_B}+\frac{1}{C_k}\big)\frac{1}{\epsilon^2} + \frac{1}{\epsilon}\sqrt{\big(\frac{1}{C_B}+\frac{1}{C_k}\big)}\Big).
	\end{align*}	
	Then, for $C_B,C_K$ large enough, we obtain from the above inequality that 
$	\mathbb{E}\|\nabla \mathcal{L}(w_\zeta)\| \leq \epsilon.$
	Thus, the total number of gradient computations is given by $B(T+NS)=\mathcal{O}(\epsilon^{-2}(T+NS)).$ Furthermore, the  total number of Hessian computations is given by $BNS =\mathcal{O}(NS\epsilon^{-2}) $
at each iteration.  
	Then, the proof is complete.

\section{Conclusion and Future Work}
In this paper, we provide a new theoretical framework for analyzing the convergence of multi-step MAML algorithm for both the resampling case and the finite-sum case. Our analysis covers most applications including reinforcement learning and supervised learning of interest. 
Our analysis reveals that 
a properly chosen inner stepsize is crucial for guaranteeing MAML to converge with the complexity increasing only linearly with $N$ (the number of the inner-stage gradient updates). 
Moreover, for problems with small  Hessians,  the inner stepsize can be set larger while maintaining the convergence. Our results also provide justifications for the empirical findings in training MAML. 

We expect that our analysis framework can be applied to understand the convergence of MAML in other scenarios such as various RL problems and Hessian-free MAML algorithms. 

\acks{The work was supported in part by the U.S. National Science Foundation under Grants CCF-1761506, 
ECCS-1818904, and CCF-1900145. }


\newpage
{\noindent \Large \bf Appendices} 
\appendix
\section{Examples for Two Types of Objective Functions}\label{ggpopssasdasdax}
\subsection{RL Example for Resampling Case}\label{apen:rlcase}
RL problems are often captured by objective functions in the expectation form. Consider a RL meta learning problem,  where each task  corresponds to  a Markov decision process (MDP) with horizon $H$. Each RL task $\mathcal{T}_i$ corresponds to an initial state distribution $\rho_i$, a policy $\pi_w$ parameterized by $w$ that denotes a distribution over the action set given each state, and a transition distribution kernel $q_i(x_{t+1}|x_t,a_t)$ at time steps $t=0,...,H-1$. Then, the loss $l_i(w)$ is defined as negative total reward, i.e., 
\begin{align*}
(\text{RL example}):\quad l_i(w):= -\mathbb{E}_{\tau\sim p_i(\cdot| w)}[ \mathcal{R}(\tau)],
\end{align*}  
where $\tau= (s_0,a_0,s_1,a_1,...,s_{H-1},a_{H-1})$ is a  trajectory following the distribution $p_i(\cdot | w)$, and the reward $$\mathcal{R}(\tau) := \sum_{t=0}^{H-1} \gamma^t\mathcal{R}(s_t,a_t)$$ with $\mathcal{R}(\cdot)$ given as a reward function. The estimated gradient here is  $$\nabla l_i(w; \Omega):= \frac{1}{|\Omega|} \sum_{\tau \in \Omega} g_i(w; \tau),$$ where $g_i(w; \tau)$ is an unbiased policy gradient estimator s.t. $\mathbb{E}_{\tau \sim p_i(\cdot|w)} g_i(w;\tau)= \nabla l_i(w)$, e.g, REINFORCE~\citep{williams1992simple} or G(PO)MDP~\citep{baxter2001infinite}. In addition, the estimated Hessian  is $$\nabla^2 l_i(w; \Omega):= \frac{1}{|\Omega|}\sum_{\tau\in \Omega}H_i(w; \tau)$$, where $H_i(w;\tau)$ is an unbiased policy Hessian estimator, e.g., DiCE~\citep{foerster2018dice} or LVC~\citep{rothfuss2019promp}.

\subsection{Classification Example for Finite-Sum Case}
The risk minimization problem in classification often has a finite-sum objective function. For example, the 
mean-squared error (MSE) loss takes the form of  
\begin{align*}
(\text{Classification example}):\quad  l_{S_i}(w):= \frac{1}{|S_i|}\sum_{(x_j, y_j)\in S_i}\|y_j - \phi (w; x_i) \|^2 \quad (\text{similarly for } \, l_{T_i}(w)), 
\end{align*}
where $x_j, y_j$ are a feature-label pair and $\phi(w;\cdot)$ can be  a deep neural network parameterized by $w$.

\section{Derivation of Simplified Form of  Gradient $\nabla \mathcal{L}_i(w)$ in~\eqref{nablaF}}\label{simplifeid}
First note that $\mathcal{L}_i(w_k) = l_i(\widetilde w_{k,N}^i)$ and $\widetilde w_{k,N}^i$ is obtained by the following gradient descent updates
\begin{align}\label{gd_pr}
\widetilde w^i_{k, j+1} =\widetilde w^i_{k,j} - \alpha \nabla l_i(\widetilde w^i_{k,j}), \,\,j = 0, 1,..., N-1 \, \text{ with }\, \widetilde w^i_{k,0} := w_k.
\end{align}
Then, by the chain rule, we have 
\begin{align*}
\nabla \mathcal{L}_i(w_k) = \nabla_{w_k} l_i(\widetilde w_{k,N}^i) = \prod_{j=0}^{N-1}\nabla_{\widetilde w_{k,j}^i} \left(\widetilde w_{k,j+1}^i\right) \nabla l_i(\widetilde w_{k,N}^i), 
\end{align*}
which, in conjunction with \eqref{gd_pr}, implies that 
\begin{align*}
\nabla \mathcal{L}_i(w_k) =\prod_{j=0}^{N-1}\nabla_{\widetilde w_{k,j}^i} \left(\widetilde w^i_{k,j} - \alpha \nabla l_i(\widetilde w^i_{k,j})\right) \nabla l_i(\widetilde w_{k,N}^i) = 
\prod_{j=0}^{N-1} \left(I - \alpha \nabla^2 l_i(\widetilde w^i_{k,j})\right) \nabla l_i(\widetilde w_{k,N}^i),
\end{align*}
which finishes the proof. 

\section{Auxiliary Lemmas for MAML in Resampling Case}
\label{aux:lemma}

In this section, we derive some useful lemmas  to prove the propositions given  in Section~\ref{se:opro} on the properties of the meta gradient and the main results Theorem~\ref{th:mainonline} and Corollary~\ref{co:online}. 

The first lemma provides a bound on the difference between $\|\widetilde w_j^i - \widetilde u_j^i\|$ for $j=0,...,N,  i\in\mathcal{I}$, where $\widetilde w_j^i,\, j=0,...,N, i\in\mathcal{I}$ are given through the {\em gradient descent} updates in~\eqref{gd_w} and  $\widetilde u_j^i,\, j=0,...,N$ are defined in the same way. 
\begin{lemma}\label{d_u_w}
	For any $i\in\mathcal{I}$, $j=0,...,N$ and $w,u \in \mathbb{R}^d$, we have 
	\begin{align*}
	\left\|\widetilde w_j^i -\widetilde  u_j^i\right\| \leq (1+\alpha L)^j \|w-u\|. 
	\end{align*}
\end{lemma}
\begin{proof}
	Based on the updates that $\widetilde w_m^i = \widetilde w_{m-1}^i - \alpha\nabla l_i(\widetilde w_{m-1}^i)$ and $\widetilde u_m^i = \widetilde u_{m-1}^i - \alpha\nabla l_i(\widetilde u_{m-1}^i)$, we obtain, for any $i \in\mathcal{I}$,
	\begin{align*}
	\|\widetilde w_m^i - \widetilde u_m^i\|  =& \|\widetilde w_{m-1}^i - \alpha\nabla l_i(\widetilde w_{m-1}^i) -\widetilde u_{m-1}^i + \alpha\nabla l_i(\widetilde u_{m-1}^i) \| \nonumber
	\\ \overset{(i)}\leq & \|\widetilde w_{m-1}^i -\widetilde u_{m-1}^i\| + \alpha L\|\widetilde w_{m-1}^i - \widetilde u_{m-1}^i\|  \nonumber
	\\ \leq & (1+\alpha L)  \|\widetilde w_{m-1}^i -\widetilde u_{m-1}^i\|,
	\end{align*} 
	where (i) follows from the triangle inequality. Telescoping the above inequality over $m$ from $1$ to $j$, we obtain 
	\begin{align*}
	\left\|\widetilde w_j^i - \widetilde u_j^i\right\| \leq (1+\alpha L)^j \|\widetilde w^i_0-\widetilde u^i_0\|, 
	\end{align*}
	which, in conjunction with the fact that $\widetilde w_0^i = w$ and $\widetilde u_0^i = u$, finishes the proof.
\end{proof}
The following lemma provides an upper bound on $\|\nabla l_i(\widetilde w_j^i)\|$ for all $i\in\mathcal{I}$ and $j=0,..., N$, where $\widetilde w_j^i$ is defined in the same way as in Lemma~\ref{d_u_w}. 
\begin{lemma}\label{le:jiw}
	For any $i\in\mathcal{I}$,  $j=0,...,N$ and $w \in \mathbb{R}^d$, we have 
	\begin{align*}
	\|\nabla l_i(\widetilde w_j^i)\| \leq (1+\alpha L)^j \|\nabla l_i(w)\|.
	\end{align*}
\end{lemma}
\begin{proof}
	For $m\geq1$, we have 
	\begin{align*}
	\|\nabla l_i(\widetilde w_m^i)\| = & \|\nabla l_i(\widetilde w_m^i) - \nabla l_i(\widetilde w_{m-1}^i) + \nabla l_i(\widetilde w_{m-1}^i)\|
	\\\leq & \|\nabla l_i(\widetilde w_m^i) - \nabla l_i(\widetilde w_{m-1}^i) \| +  \|\nabla l_i(\widetilde w_{m-1}^i)\|
	\\\leq & L\|\widetilde w_m^i - \widetilde w_{m-1}^i\| +\|\nabla l_i(\widetilde w_{m-1}^i)\|\leq (1+\alpha L)\|\nabla l_i(\widetilde w_{m-1}^i)\|,
	\end{align*}
	where the last inequality follows from the update $\widetilde w_m^i =  \widetilde w_{m-1}^i - \alpha \nabla l_i(\widetilde w_{m-1}^i)$. Then, telescoping the above inequality over $m$ from $1$ to $j$ yields
	\begin{align*}
	\|\nabla l_i(\widetilde w_j^i)\| \leq (1+\alpha L)^j \|\nabla l_i(\widetilde w_0^i)\|,
	\end{align*}
	which, combined with the fact that $\widetilde w_0^i = w$, finishes the proof.  
\end{proof}
The following lemma gives an upper bound on  the quantity $\big\|I - \prod_{j=0}^m(I - \alpha V_j)\big\|$ for all matrices  $V_j \in \mathbb{R}^{d\times d},j=0,...,m$ that satisfy $\|V_j\|\leq L$.
\begin{lemma}\label{le:prd}
	For all matrices $V_j \in \mathbb{R}^{d\times d}, j =0,..., m$ that satisfy $\|V_j\|\leq L$, we have 
	\begin{align*}
	\Big\|I - \prod_{j=0}^m(I - \alpha V_j)\Big\| \leq (1+\alpha L)^{m+1} - 1.
	\end{align*}
\end{lemma}
\begin{proof}
	First note that the product $ \prod_{j=0}^m(I - \alpha V_j) $ can be expanded as  
	\begin{align}
	\prod_{j=0}^m(I - \alpha V_j) = I - \sum_{j=0}^m \alpha V_j + \sum_{0\leq p < q \leq m} \alpha^2 V_p V_q+\cdots+(-1)^{m+1} \alpha^{m+1}\prod_{j=0}^m V_j.\nonumber
	\end{align}
	Then, by using $\|V_j\|\leq L$ for $j=0,...,m$, we have 
	\begin{align}
	\Big\|I - \prod_{j=0}^m(I - \alpha V_j)\Big\|  \leq & \Big\|\sum_{j=0}^m \alpha V_j \Big\| + \Big\|\sum_{0\leq p < q \leq m} \alpha^2 V_pV_q \Big\| + \cdots + \Big\| \alpha^{m+1}\prod_{j=0}^m V_j\Big\| \nonumber
	\\\leq & {\rm C}^1_{m+1} \alpha L + {\rm C}_{m+1}^2 (\alpha L)^2 + \cdots + {\rm C}_{m+1}^{m+1} (\alpha L)^{m+1}\nonumber
	\\ = & (1+\alpha L)^{m+1} - 1, \nonumber
	\end{align}
	where the notion $C_n^k$ denotes the number of $k$-element subsets of a set of size $n$. Then, the proof is complete. 
\end{proof}
Recall the gradient { $\nabla \mathcal{L}_i(w) = \prod_{j=0}^{N-1}(I-\alpha \nabla^2 l_i(\widetilde w^i_{j}))\nabla l_i( \widetilde w^i _{N})$}, where {  $ \widetilde w_{j}^i, i\in \mathcal{I}, j=0,..., N$} are given by the gradient descent steps in \eqref{gd_w} and $\widetilde w_{0}^i = w$ for all tasks $i \in \mathcal{I}$.  
Next, we provide an upper bound on the difference { $\|\nabla l_i(w) - \nabla \mathcal{L}_i(w)\|$}.
\begin{lemma}\label{le:fF}
	For any $i \in\mathcal{I}$ and $w \in \mathbb{R}^d$, we have 
	\begin{align*}
	\|\nabla l_i(w) - \nabla \mathcal{L}_i(w)\| \leq C_l \|\nabla l_i(w)\|,
	\end{align*}
	where $C_l$ is a positive constant given by 
	\begin{align}\label{eq:cfn}
	C_l = (1+\alpha L)^{2N} - 1 > 0.
	\end{align}
\end{lemma}
\begin{proof}
	First note that $\widetilde w_N^i$ can be rewritten as $\widetilde w_N^i = w - \alpha \sum_{j=0}^{N-1} \nabla l_i\big(\widetilde w_j^i\big)$. Then, based on  the mean value theorem (MVT) for vector-valued functions~\citep{mcleod1965mean}, we have, there exist constants $r_t, t=1,...,d$ satisfying $\sum_{t=1}^d r_t =1$ and vectors $w_t^\prime\in\mathbb{R}^d, t=1,...,d$ such that     
	\begin{align}\label{mvts}
	\nabla l_i( \widetilde w^i _{N}) =& \nabla l_i\Big( w - \alpha \sum_{j=0}^{N-1} \nabla l_i\big(\widetilde w_j^i\big)\Big)= \nabla l_i(w) + \Big(\sum_{t=1}^dr_t\nabla^2 l_i (w_t^\prime)\Big) \Big(-\alpha \sum_{j=0}^{N-1} \nabla l_i\big(\widetilde w_j^i\big)\Big) \nonumber
	\\ = &  \Big(I- \alpha\sum_{t=1}^dr_t\nabla^2 l_i (w_t^\prime)\Big)\nabla l_i(w) - \alpha \sum_{t=1}^dr_t\nabla^2 l_i (w_t^\prime) \sum_{j=1}^{N-1} \nabla l_i\big(\widetilde w_j^i\big).
	\end{align}
	For simplicity, we define  $K(N):=  \prod_{j=0}^{N-1}(I-\alpha \nabla^2 l_i(\widetilde w^i_{j}))$. Then, using~\eqref{mvts},  we write $	\|\nabla l_i(w) - \nabla \mathcal{L}_i(w)\|$ as 
	\begin{align}
	\|\nabla l_i(w) -& \nabla \mathcal{L}_i(w)\| = \|\nabla l_i(w) - K(N)\nabla l_i(\widetilde w_N^i)\| \nonumber
	\\ =& \Big\|\nabla l_i(w) - K(N)\Big(I- \alpha\sum_{t=1}^dr_t\nabla^2 l_i (w_t^\prime)\Big)\nabla l_i(w)  + \alpha K(N)\sum_{t=1}^dr_t\nabla^2 l_i (w_t^\prime)\sum_{j=1}^{N-1} \nabla l_i\big(\widetilde w_j^i\big)\Big\| \nonumber
	\\ \leq & \Big\|\Big(I - K(N)\Big(I- \alpha\sum_{t=1}^dr_t\nabla^2 l_i (w_t^\prime)\Big)\Big)\nabla l_i(w) \Big\| + \Big\| \alpha K(N)\sum_{t=1}^dr_t\nabla^2 l_i (w_t^\prime)\sum_{j=1}^{N-1} \nabla l_i\big(\widetilde w_j^i\big)\Big\| \nonumber
	\\ \overset{(i)}\leq & \Big\|\Big(I - K(N)\Big(I- \alpha\sum_{t=1}^dr_t\nabla^2 l_i (w_t^\prime)\Big)\Big)\nabla l_i(w) \Big\| + \alpha L (1+\alpha L)^N \sum_{j=1}^{N-1}  \Big\| \nabla l_i\big(\widetilde w_j^i\big)\Big\| \nonumber
	\\\overset{(ii)}\leq & \Big\|I - K(N)\Big(I- \alpha\sum_{t=1}^dr_t\nabla^2 l_i (w_t^\prime)\Big)\Big\|\|\nabla l_i(w)\| + \alpha L (1+\alpha L)^N \sum_{j=1}^{N-1}  (1+\alpha L)^j \|\nabla l_i(w)\| \nonumber
	\\\leq & \Big\|I - K(N)\Big(I- \alpha\sum_{t=1}^dr_t\nabla^2 l_i (w_t^\prime)\Big)\Big\|\|\nabla l_i(w)\| +(1+\alpha L)^{N+1} ((1+\alpha L)^{N-1}-1)\|\nabla l_i(w)\| \nonumber
	\\\overset{(iii)}\leq & ((1+\alpha L)^{N+1}-1)\|\nabla l_i(w)\| + (1+\alpha L)^{N+1} ((1+\alpha L)^{N-1}-1)\|\nabla l_i(w)\|  \nonumber
	\\ = & ((1+\alpha L)^{2N} - 1)\|\nabla l_i(w)\|,\nonumber
	\end{align}
	where (i) follows from the fact that $\|\nabla^2 l_i(u)\| \leq L$ for any $u\in \mathbb{R}^d$ and $\sum_{t=1}^d r_t =1$,  and the inequality that $\|\sum_{j=1} ^n a_j\|\leq \sum_{j=1} ^n\|a_j\|$,  (ii) follows from  Lemma~\ref{le:jiw}, and (iii) follows from Lemma~\ref{le:prd}.
\end{proof}

Recall that the expected value of the gradient of the loss $\nabla l(w):=\mathbb{E}_{i\sim p(\mathcal{T})} \nabla l_i(w)$ and the objective function $\nabla \mathcal{L}(w): = \nabla \mathcal{L}_i(w)$. Based on the above lemmas, we next provide an upper bound on $\|\nabla l(w)\|$ using $\|\nabla \mathcal{L}(w)\|$. 
\begin{lemma}\label{le:lL}
	For any $w\in\mathbb{R}^d$, we have 
	\begin{align*}
	\|\nabla l(w)\| \leq \frac{1}{1-C_l} \|\nabla \mathcal{L}(w)\| + \frac{C_l}{1-C_l} \sigma,
	\end{align*}
	where the constant $C_l$ is given by 
     \begin{align*}
	C_l = (1+\alpha L)^{2N} - 1.
	\end{align*}
\end{lemma}
\begin{proof}
	Based on the definition of $\nabla l(w)$, we have 
	\begin{align}
	\|\nabla l(w)\| =& \|\mathbb{E}_{i\sim p(\mathcal{T})} (\nabla l_i(w) -\nabla \mathcal{L}_i(w) +\nabla \mathcal{L}_i(w) )\| \nonumber
	\\\leq& \|\mathbb{E}_{i\sim p(\mathcal{T})} \nabla \mathcal{L}_i(w) \| + \|\mathbb{E}_{i\sim p(\mathcal{T})} (\nabla l_i(w) -\nabla \mathcal{L}_i(w)  ) \|   \nonumber
	\\\leq & \|\nabla \mathcal{L}(w) \| + \mathbb{E}_{i\sim p(\mathcal{T})} \|\nabla l_i(w) -\nabla \mathcal{L}_i(w)   \|   \nonumber
	\\\overset{(i)}\leq & \|\nabla \mathcal{L}(w) \| + C_l\mathbb{E}_{i\sim p(\mathcal{T})}\|\nabla l_i(w)   \|  \nonumber
	\\ \overset{(ii)} \leq & \|\nabla \mathcal{L}(w) \| +  C_l(\|\nabla l(w)   \| + \sigma), \nonumber
	\end{align}
	where (i) follows from Lemma~\ref{le:fF}, and (ii) follows from Assumption~\ref{a2}. Then, rearranging the above inequality completes the proof.  
\end{proof}
Recall from~\eqref{hatlw} that we choose the meta stepsize $\beta_k = \frac{1}{C_\beta \widehat L_{w_k}} $, where $C_\beta$ is a positive constant and { $ \widehat L_{w_k} = (1+\alpha L)^{2N}L + C_\mathcal{L} \frac{1}{|B_k^\prime|}\sum_{i\in B_k^\prime}\|\nabla l_i(w_k; D_{L_k}^i)\|$}. Using an approach similar to Lemma 4.11 in~\cite{fallah2020convergence}, we establish the following lemma to provide the first- and second-moment bounds for $\beta_k$.
%
\begin{lemma}\label{le:xiaodege} 
Suppose that  Assumptions~\ref{assum:smooth},~\ref{a2} and~\ref{a3} hold. 
	Set the meta stepsize $\beta_k = \frac{1}{C_\beta \widehat L_{w_k}} $ with  $\widehat L_{w_k}$  given by~\eqref{hatlw}, where  $|B_k^\prime| > \frac{4C^2_{\mathcal{L}}\sigma^2}{3(1+\alpha L)^{4N}L^2}$ and $|D_{L_k}^i| > \frac{64\sigma^2_g C_\mathcal{L}^2}{(1+\alpha L)^{4N}L^2}$ for all $i \in B_k^\prime$. Then, conditioning on $w_k$, we have 
\begin{align*}
\mathbb{E} \beta_k  \geq  \frac{4}{C_\beta} \frac{1}{5 L_{w_k}},\quad \mathbb{E}\beta^2_k  \leq  \frac{4}{C_\beta^2} \frac{1}{L^2_{w_k}}, 
\end{align*}
where $L_{w_k} = (1+\alpha L)^{2N}L + C_\mathcal{L} \mathbb{E}_{i\sim p(\mathcal{T})}\|\nabla l_i(w_k)\|$ with  $C_\mathcal{L}$ given in~\eqref{clcl}. 
\end{lemma}

\begin{proof}
	Let $\widetilde L_{w_k} = 4L +\frac{4C_\mathcal{L}}{(1+\alpha L)^{2N}} \frac{1}{|B_k^\prime|}\sum_{i\in B_k^\prime}\|\nabla l_i(w_k;  D_{L_k}^i)\|$. Note that 
	  $|B_k^\prime| > \frac{4C^2_{\mathcal{L}}\sigma^2}{3(1+\alpha L)^{4N}L^2}$ and $|D_{L_k}^i| > \frac{64\sigma^2_g C_\mathcal{L}^2}{(1+\alpha L)^{4N}L^2}, \,i \in B_k^\prime$. Then, using an approach similar to (61) in \cite{fallah2020convergence} and conditioning on $w_k$, we have 
	\begin{align}\label{1toL}
	\mathbb{E} \Big(  \frac{1}{\widetilde L^2_{w_k} }\Big) \leq \frac{\sigma_\beta^2/(4L)^2 + \mu_\beta^2/(\mu_\beta)^2}{\sigma_\beta^2 + \mu_\beta^2},
	\end{align}
	where $\sigma^2_\beta$ and $\mu_\beta$ are the variance and mean of variable $\frac{4C_\mathcal{L}}{(1+\alpha L)^{2N}} \frac{1}{|B_k^\prime|}\sum_{i\in B_k^\prime}\|\nabla l_i(w_k; D_{L_k}^i)\|$. Using an approach similar to (62) in \cite{fallah2020convergence}, conditioning on $w_k$ and using $|D_{L_k}^i| > \frac{64\sigma^2_g C_\mathcal{L}^2}{(1+\alpha L)^{4N}L^2}$, we have 
	\begin{align}\label{lluu}
	\frac{C_\mathcal{L}}{(1+\alpha L)^{2N}} \mathbb{E}_i\|\nabla l_i(w_k)\| - L \leq \mu_\beta  \leq \frac{C_\mathcal{L}}{(1+\alpha L)^{2N}} \mathbb{E}_i\|\nabla l_i(w_k)\| + L,
	\end{align}
	which implies that $\mu_\beta +5L \geq \frac{4}{(1+\alpha L)^{2N}}L_{w_k}$, and thus using \eqref{1toL} yields
	\begin{align}\label{lopops}
	\frac{16}{(1+\alpha L)^{4N}}L^2_{w_k}	\mathbb{E} \Big(  \frac{1}{\widetilde L^2_{w_k} }\Big) \leq \frac{\mu_\beta^2(25/16+ \sigma_\beta^2/(8L^2))+25\sigma_\beta^2/8}{\sigma_\beta^2 + \mu_\beta^2 }.
	\end{align}
	Furthermore, conditioning on $w_k$,  $\sigma_\beta$ is bounded by 
	\begin{align}\label{signbeta}
	\sigma_\beta^2 =& \frac{16 C^2_\mathcal{L}}{(1+\alpha L)^{4N}|B^\prime_k|} \text{Var} (\|\nabla l_i(w_k; D_{L_k}^i)\|) \nonumber
	\\ \leq & \frac{16 C^2_\mathcal{L}}{(1+\alpha L)^{4N}|B^\prime_k|} \Big(\sigma^2 + \frac{\sigma_g^2}{|D_{L_k}^i|}\Big) \nonumber
	\\\overset{(i)}\leq & \frac{16 C^2_\mathcal{L}\sigma^2}{(1+\alpha L)^{4N}|B^\prime_k|}  + \frac{L^2}{4|B_k^\prime|} \overset{(ii)}\leq  12L^2 + \frac{1}{4}L^2 < \frac{25}{2} L^2,
	\end{align} 
	where (i) follows from $|D_{L_k}^i| > \frac{64\sigma^2_g C_\mathcal{L}^2}{(1+\alpha L)^{4N}L^2}, \,i \in B_k^\prime$ and (ii) follows from $|B_k^\prime| > \frac{4C^2_{\mathcal{L}}\sigma^2}{3(1+\alpha L)^{4N}L^2}$ and $|B_k^\prime| \geq 1$. Then, plugging \eqref{signbeta} in~\eqref{lopops} yields
$	\frac{16}{(1+\alpha L)^{4N}}L^2_{w_k}	\mathbb{E} \Big(  \frac{1}{\widetilde L^2_{w_k} }\Big) \leq \frac{25}{8}.$
	Then, noting that $\beta_k  = \frac{4}{C_\beta (1+\alpha L)^{2N} \widetilde L_{w_k}}$, using the above inequality and conditioning on $w_k$,  we have 
	\begin{align}\label{secondmm}
	\mathbb{E}\beta^2_k =  \frac{16}{C^2_\beta (1+\alpha L)^{4N}}\mathbb{E}  \left(\frac{1}{ \widetilde L^2_{w_k}} \right) \leq \frac{25}{8C_\beta^2} \frac{1}{L^2_{w_k}} < \frac{4}{C_\beta^2} \frac{1}{L^2_{w_k}}.
	\end{align}
	In addition, by Jensen's inequality and conditioning on $w_k$, we have 
	\begin{align}\label{onemo}
	\mathbb{E} \beta_k =&\frac{4}{C_\beta (1+\alpha L)^{2N} } \mathbb{E}\Big(\frac{1}{ \widetilde L_{w_k}}\Big) \geq  \frac{4}{C_\beta (1+\alpha L)^{2N} } \frac{1}{ \mathbb{E}\widetilde L_{w_k}} =  \frac{4}{C_\beta (1+\alpha L)^{2N} } \frac{1}{4L + \mu_\beta} \nonumber
	\\\overset{(i)}\geq &   \frac{4}{C_\beta  } \frac{1}{4L(1+\alpha L)^{2N} +L_{w_k}}  \overset{(ii)}\geq \frac{4}{C_\beta} \frac{1}{5 L_{w_k}}, 
	\end{align}
	where (i) follows from \eqref{lluu} and (ii) follows from the fact $L_{w_k} >  (1+\alpha L)^{2N}L$.
\end{proof}

\section{Auxiliary Lemmas for MAML in Finite-Sum Case}\label{aux:lemma_finite}
In this section, we provide some useful lemmas to prove the propositions in Section~\ref{mainsec:off} on properties of the meta gradient and the main results Theorem~\ref{mainth:offline} and Corollary~\ref{co:mainoffline}.

The following lemma provides an upper bound on $\|l_{S_i}(\widetilde w^i_j)\|$  for all $i\in\mathcal{I}$ and $j=0,..., N$, where $\widetilde w_j^i$ is defined by~\eqref{innerfinite} with $\widetilde w_0^i=w$. 
\begin{lemma}\label{finite:gbd}
	For any $i\in\mathcal{I}$,  $j=0,...,N$ and $w \in \mathbb{R}^d$, we have 
	\begin{align*}
	\|\nabla l_{S_i}(\widetilde w_j^i)\| \leq (1+\alpha L)^j \|\nabla l_{S_i}(w)\|.
	\end{align*}
\end{lemma}
\begin{proof}
	The proof is similar to that of Lemma~\ref{le:jiw}, and thus omitted.  
\end{proof}
We next provide a bound on  $\|\nabla l_{T_i}(w) - \nabla \mathcal{L}_i(w) \|$, where $$\nabla \mathcal{L}_i(w) = \prod_{j=0}^{N-1}(I - \alpha \nabla^2 l_{S_i}(w_{j}^i))\nabla l_{T_i}(w_{N}^i).$$
\begin{lemma}\label{tiis}
	For any $i \in\mathcal{I}$ and $w \in \mathbb{R}^d$, we have 
	\begin{align*}
	\|\nabla l_{T_i}(w) - \nabla \mathcal{L}_i(w)\| \leq \big( (1+\alpha L)^N -1  \big)\|\nabla l_{T_i}(w)\| + (1+\alpha L)^N  \big( (1+\alpha L)^N -1  \big) \|\nabla l_{S_i}(w)\|.
	\end{align*}
\end{lemma}
\begin{proof}
	Using the mean value theorem (MVT), we have, there exist constants $r_t, t=1,...,d$ satisfying $\sum_{t=1}^d r_t =1$ and vectors $w_t^\prime\in\mathbb{R}^d, t=1,...,d$ such that     

	\begin{align*}
	\nabla l_{T_i}(\widetilde w_N^i) = & \nabla  l_{T_i} \Big(w - \alpha \sum_{j=0}^{N-1} \nabla l_{S_i}(\widetilde w_j^i) \Big) = \nabla l_{T_i}(w) + \sum_{t=1}^dr_t \nabla^2  l_{T_i} (w_t^\prime)\Big(  - \alpha \sum_{j=0}^{N-1} \nabla l_{S_i}(\widetilde w_j^i)      \Big) \nonumber
	\\ = &\nabla l_{T_i}(w) - \alpha\sum_{t=1}^dr_t \nabla^2  l_{T_i} (w_t^\prime)\sum_{j=0}^{N-1} \nabla l_{S_i}(\widetilde w_j^i).    
	\end{align*}
	Based on the above equality, we have
	\begin{small}
	\begin{align}
	\|\nabla &l_{T_i}(w) - \nabla \mathcal{L}_i(w)\| \nonumber
	\\=& \Big\|\nabla l_{T_i}(w) - \prod_{j=0}^{N-1}(I - \alpha \nabla^2 l_{S_i}(\widetilde w_{j}^i))\nabla l_{T_i}(\widetilde w_{N}^i)\Big\| \nonumber
	\\ =& \Big\|\nabla l_{T_i}(w) - \prod_{j=0}^{N-1}(I - \alpha \nabla^2 l_{S_i}(\widetilde w_{j}^i))\nabla l_{T_i}(w) + \prod_{j=0}^{N-1}(I - \alpha \nabla^2 l_{S_i}(\widetilde w_{j}^i)) \alpha\sum_{t=1}^dr_t \nabla^2  l_{T_i} (w_t^\prime) \sum_{j=0}^{N-1} \nabla l_{S_i}(\widetilde w_j^i)\Big\| \nonumber 
	\\ =& \Big\|I - \prod_{j=0}^{N-1}(I - \alpha \nabla^2 l_{S_i}(\widetilde w_{j}^i)) \Big\|\|\nabla l_{T_i}(w)\|+ \Big\| \prod_{j=0}^{N-1}(I - \alpha \nabla^2 l_{S_i}(\widetilde w_{j}^i)) \alpha\sum_{t=1}^dr_t \nabla^2  l_{T_i} (w_t^\prime) \sum_{j=0}^{N-1} \nabla l_{S_i}(\widetilde w_j^i)\Big\| \nonumber 
	\\ \overset{(i)}\leq & \big((1+\alpha L)^N -1\big) \|\nabla l_{T_i}(w)\| + \alpha L(1+\alpha L)^N\sum_{j=0}^{N-1} \|\nabla l_{S_i}(\widetilde w_j^i)\| \nonumber
	\\\overset{(ii)} \leq & \big((1+\alpha L)^N -1\big) \|\nabla l_{T_i}(w)\| + \alpha L(1+\alpha L)^N\sum_{j=0}^{N-1} (1+\alpha L)^j\|\nabla l_{S_i}(w)\| \nonumber
	\\ = & \big((1+\alpha L)^N -1\big) \|\nabla l_{T_i}(w)\| + (1+\alpha L)^N \big((1+\alpha L)^N-1\big)\|\nabla l_{S_i}(w)\|, \nonumber
	\end{align}
	\end{small}
	\hspace{-0.12cm}where (i) follows from Lemma~\ref{le:prd} and $\|\sum_{t=1}^dr_t \nabla^2  l_{T_i} (w_t^\prime)\|\leq \sum_{t=1}^dr_t\| \nabla^2  l_{T_i} (w_t^\prime)\|\leq L$, and (ii) follows from Lemma~\ref{finite:gbd}. Then, the proof is complete. 
\end{proof}	
Recall that $\nabla l_T(w) = \mathbb{E}_{i\sim p(\mathcal{T})} \nabla l_{T_i}(w)$, $\nabla \mathcal{L}(w) =  \mathbb{E}_{i\sim p(\mathcal{T})} \nabla \mathcal{L}_i(w)$ and $b = \mathbb{E}_{i\sim p(\mathcal{T})} [b_i]$. 
The following lemma provides an upper bound on $\|\nabla l_T(w)\|$.	
\begin{lemma}\label{twc1c2}
	For any $i \in\mathcal{I}$ and $w \in \mathbb{R}^d$, we have 
	\begin{align}
	\|\nabla l_T(w)\| \leq \frac{1}{C_1} \|\nabla \mathcal{L}(w)\| + \frac{C_2}{C_1},
	\end{align}
	where constants $C_1, C_2>0$ are give by 
	\begin{align}\label{c1c2}
	C_1 =& 2-(1+\alpha L)^{2N}, \nonumber
	\\C_2 =& \big( (1+\alpha L)^{2N}-1  \big)\sigma + (1+\alpha L)^N \big((1+\alpha L)^N -1 \big) b.
	\end{align}
\end{lemma}	
\begin{proof}
	First note that 
	\begin{align*}
	\|\nabla l_T(w)\| =& \|\mathbb{E}_i (\nabla l_{T_i}(w) -\nabla \mathcal{L}_i(w) ) +\nabla \mathcal{L}(w)\|  \nonumber
	\\ \leq & \|\nabla \mathcal{L}(w)\| + \mathbb{E}_i \| \nabla l_{T_i}(w) -\nabla \mathcal{L}_i(w) \| \nonumber
	\\ \overset{(i)}\leq &  \|\nabla \mathcal{L}(w)\| + \mathbb{E}_i  \Big(  \big( (1+\alpha L)^N -1  \big)\|\nabla l_{T_i}(w)\| + (1+\alpha L)^N  \big( (1+\alpha L)^N -1  \big) \|\nabla l_{S_i}(w)\|  \Big) \nonumber
	\\ \overset{(ii)}\leq &  \|\nabla \mathcal{L}(w)\| +   \big( (1+\alpha L)^N -1  \big)\big( \|\nabla l_{T}(w)\| +\sigma \big) 
	\\&+ (1+\alpha L)^N  \big( (1+\alpha L)^N -1  \big) (\mathbb{E}_i \|\nabla l_{T_i}(w)\| + \mathbb{E}_ib_i) \nonumber
	\\ \leq & \|\nabla \mathcal{L}(w)\| +   \big( (1+\alpha L)^N -1 + (1+\alpha L)^N((1+\alpha L)^N-1) \big) \|\nabla l_{T}(w)\|  \nonumber
	\\ &+ ((1+\alpha L)^N -1)\sigma + (1+\alpha L)^N((1+\alpha L)^N -1)(\sigma + b) \nonumber
	\\ \leq & \|\nabla \mathcal{L}(w)\| +   \big( (1+\alpha L)^{2N} -1 \big) \|\nabla l_{T}(w)\|  
	\\&+ ((1+\alpha L)^{2N} -1)\sigma + (1+\alpha L)^N((1+\alpha L)^N -1)b
	\end{align*}
	where (i) follows from Lemma~\ref{tiis}, (ii) follows from Assumption~\ref{assum:vaoff}. Based on the definitions of $C_1$ and $C_2$ in~\eqref{c1c2}, the proof is complete. 
\end{proof}

The following lemma provides the first- and second-moment bounds on $1/\hat L_{w_k}$, where 
\begin{align*}
\hat L_{w_k} =(1+\alpha L)^{2N}L + C_b b +  C_\mathcal{L}\frac{\sum_{i \in B_k^\prime}\|\nabla l_{T_i}(w_k)\|}{|B_k^\prime|}.
\end{align*}
\begin{lemma}\label{le:betak}
	If the batch size $|B_k^\prime| \geq \frac{2C^2_\mathcal{L}\sigma^2}{( C_b b + (1+\alpha L)^{2N} L)^2}$, then conditioning on $w_k$, we have
	\begin{align*}
	\mathbb{E} \Big( \frac{1}{\hat L_{w_k}}  \Big) \geq \frac{1}{L_{w_k}}, \quad \mathbb{E} \Big( \frac{1}{\hat L^2_{w_k}}  \Big) \leq \frac{2}{L^2_{w_k}} 
	\end{align*}
	where $L_{w_k}$ is given by $$L_{w_k}= (1+\alpha L)^{2N}L + C_b b +  C_\mathcal{L} \mathbb{E}_{i\sim p(\mathcal{T})}\|\nabla l_{T_i}(w_k)\|.$$
\end{lemma}
\begin{proof}
	Conditioning on $w_k$ and using an approach similar to \eqref{1toL}, we have 	
	\begin{align}
	\mathbb{E} \Big( \frac{1}{\hat L^2_{w_k}}  \Big) \leq \frac{\sigma_\beta^2 / \big(  C_b b + (1+\alpha L)^{2N} L  \big)^2 + \mu^2_\beta / (\mu_\beta + C_b b + (1+\alpha L)^{2N} L)^2}{\sigma_\beta^2 + \mu_\beta^2},
	\end{align}
	where $\mu_\beta$ and $\sigma^2_\beta$ are the  mean and variance of variable $\frac{C_\mathcal{L}}{|B_k^\prime|}\sum_{i \in B_k^\prime}\|\nabla l_{T_i}(w_k)\|$. Noting that $\mu_\beta = C_\mathcal{L} \mathbb{E}_{i\sim p(\mathcal{T})}\|\nabla l_{T_i}(w_k)\|$, we have $L_{w_k} = (1+\alpha L)^{2N}L + C_b b +  \mu_\beta$, and thus 
	\begin{align}\label{halfgo}
	L^2_{w_k}\mathbb{E} \Big( \frac{1}{\hat L^2_{w_k}}  \Big) \leq \frac{\sigma_\beta^2\frac{((1+\alpha L)^{2N}L + C_b b +  \mu_\beta)^2}{ \big(  C_b b + (1+\alpha L)^{2N} L  \big)^2 }+ \mu^2_\beta }{\sigma_\beta^2 + \mu_\beta^2} \leq \frac{2\sigma_\beta^2+\mu_\beta^2 + \frac{2\sigma_\beta^2\mu_\beta^2}{ \big(  C_b b + (1+\alpha L)^{2N} L  \big)^2}}{\sigma_\beta^2 + \mu_\beta^2},
	\end{align}
	where the last inequality follows from $(a+b)^2\leq 2a^2+2b^2$.
	Note that, conditioning on $w_k$, 
	\begin{align*}
	\sigma_\beta^2 = \frac{C^2_\mathcal{L} }{|B_k^\prime|}    \text{Var} \|\nabla l_{T_i}(w_k)\| \leq  \frac{C^2_\mathcal{L} }{|B_k^\prime|}   \sigma^2,
	\end{align*}
	which, in conjunction with $|B_k^\prime| \geq \frac{2C^2_\mathcal{L}\sigma^2}{( C_b b + (1+\alpha L)^{2N} L)^2}$, yields
	\begin{align}\label{sigbbs}
	\frac{2\sigma_\beta^2}{ \big(  C_b b + (1+\alpha L)^{2N} L  \big)^2} \leq 1.
	\end{align}
	Combining~\eqref{sigbbs} and~\eqref{halfgo} yields $$\mathbb{E} \Big( \frac{1}{\hat L^2_{w_k}}  \Big) \leq \frac{2}{L^2_{w_k}}. $$ 
	In addition, conditioning on $w_k$, we have
	\begin{align}
	\mathbb{E} \Big( \frac{1}{\hat L_{w_k}}  \Big)  \overset{(i)}\geq\frac{1}{	\mathbb{E}  \hat L_{w_k}}  = \frac{1}{L_{w_k}},
	\end{align}
	where (i) follows from  Jensen's inequality. Then, the proof is complete. 
\end{proof}


\vskip 0.2in

\bibliography{ref}

\begin{thebibliography}{50}
\providecommand{\natexlab}[1]{#1}
\providecommand{\url}[1]{\texttt{#1}}
\expandafter\ifx\csname urlstyle\endcsname\relax
  \providecommand{\doi}[1]{doi: #1}\else
  \providecommand{\doi}{doi: \begingroup \urlstyle{rm}\Url}\fi

\bibitem[Al-Shedivat et~al.(2018)Al-Shedivat, Bansal, Burda, Sutskever,
  Mordatch, and Abbeel]{al2018continuous}
Maruan Al-Shedivat, Trapit Bansal, Yuri Burda, Ilya Sutskever, Igor Mordatch,
  and Pieter Abbeel.
\newblock Continuous adaptation via meta-learning in nonstationary and
  competitive environments.
\newblock In \emph{International Conference on Learning Representations
  (ICLR)}, 2018.

\bibitem[Alquier et~al.(2017)Alquier, Pontil, et~al.]{alquier2017regret}
Pierre Alquier, Massimiliano Pontil, et~al.
\newblock Regret bounds for lifelong learning.
\newblock In \emph{Artificial Intelligence and Statistics (AISTATS)}, pages
  261--269, 2017.

\bibitem[Antoniou et~al.(2019)Antoniou, Edwards, and
  Storkey]{antoniou2019train}
Antreas Antoniou, Harrison Edwards, and Amos Storkey.
\newblock How to train your {MAML}.
\newblock In \emph{International Conference on Learning Representations
  (ICLR)}, 2019.

\bibitem[Arora et~al.(2020)Arora, Du, Kakade, Luo, and
  Saunshi]{arora2020provable}
Sanjeev Arora, Simon~S Du, Sham Kakade, Yuping Luo, and Nikunj Saunshi.
\newblock Provable representation learning for imitation learning via bi-level
  optimization.
\newblock In \emph{International conference on machine learning (ICML)}, 2020.

\bibitem[Balcan et~al.(2019)Balcan, Khodak, and Talwalkar]{balcan2019provable}
Maria-Florina Balcan, Mikhail Khodak, and Ameet Talwalkar.
\newblock Provable guarantees for gradient-based meta-learning.
\newblock In \emph{International Conference on Machine Learning (ICML)}, pages
  424--433, 2019.

\bibitem[Baxter and Bartlett(2001)]{baxter2001infinite}
Jonathan Baxter and Peter~L Bartlett.
\newblock Infinite-horizon policy-gradient estimation.
\newblock \emph{Journal of Artificial Intelligence Research}, 15:\penalty0
  319--350, 2001.

\bibitem[Bengio et~al.(1991)Bengio, Bengio, and Cloutier]{bengio1991learning}
Y~Bengio, S~Bengio, and J~Cloutier.
\newblock Learning a synaptic learning rule.
\newblock In \emph{International Joint Conference on Neural Networks (IJCNN)}.
  IEEE, 1991.

\bibitem[Chen et~al.(2018)Chen, Dong, Li, and He]{chen2018federated}
Fei Chen, Zhenhua Dong, Zhenguo Li, and Xiuqiang He.
\newblock Federated meta-learning for recommendation.
\newblock \emph{arXiv preprint arXiv:1802.07876}, 2018.

\bibitem[Collins et~al.(2020)Collins, Mokhtari, and
  Shakkottai]{collins2020distribution}
Liam Collins, Aryan Mokhtari, and Sanjay Shakkottai.
\newblock Distribution-agnostic model-agnostic meta-learning.
\newblock \emph{arXiv preprint arXiv:2002.04766}, 2020.

\bibitem[Denevi et~al.(2018{\natexlab{a}})Denevi, Ciliberto, Stamos, and
  Pontil]{denevi2018incremental}
Giulia Denevi, Carlo Ciliberto, Dimitris Stamos, and Massimiliano Pontil.
\newblock Incremental learning-to-learn with statistical guarantees.
\newblock \emph{arXiv preprint arXiv:1803.08089}, 2018{\natexlab{a}}.

\bibitem[Denevi et~al.(2018{\natexlab{b}})Denevi, Ciliberto, Stamos, and
  Pontil]{denevi2018learning}
Giulia Denevi, Carlo Ciliberto, Dimitris Stamos, and Massimiliano Pontil.
\newblock Learning to learn around a common mean.
\newblock In \emph{Advances in Neural Information Processing Systems
  (NeurIPS)}, pages 10169--10179, 2018{\natexlab{b}}.

\bibitem[Denevi et~al.(2019)Denevi, Ciliberto, Grazzi, and
  Pontil]{denevi2019learning}
Giulia Denevi, Carlo Ciliberto, Riccardo Grazzi, and Massimiliano Pontil.
\newblock Learning-to-learn stochastic gradient descent with biased
  regularization.
\newblock \emph{arXiv preprint arXiv:1903.10399}, 2019.

\bibitem[Du et~al.(2020)Du, Hu, Kakade, Lee, and Lei]{du2020few}
Simon~S Du, Wei Hu, Sham~M Kakade, Jason~D Lee, and Qi~Lei.
\newblock Few-shot learning via learning the representation, provably.
\newblock \emph{arXiv preprint arXiv:2002.09434}, 2020.

\bibitem[Fallah et~al.(2020{\natexlab{a}})Fallah, Mokhtari, and
  Ozdaglar]{fallah2020convergence}
Alireza Fallah, Aryan Mokhtari, and Asuman Ozdaglar.
\newblock On the convergence theory of gradient-based model-agnostic
  meta-learning algorithms.
\newblock In \emph{International Conference on Artificial Intelligence and
  Statistics (AISTATS)}, pages 1082--1092, 2020{\natexlab{a}}.

\bibitem[Fallah et~al.(2020{\natexlab{b}})Fallah, Mokhtari, and
  Ozdaglar]{fallah2020provably}
Alireza Fallah, Aryan Mokhtari, and Asuman Ozdaglar.
\newblock Provably convergent policy gradient methods for model-agnostic
  meta-reinforcement learning.
\newblock \emph{arXiv preprint arXiv:2002.05135}, 2020{\natexlab{b}}.

\bibitem[Finn and Levine(2018)]{finn2018meta}
Chelsea Finn and Sergey Levine.
\newblock Meta-learning and universality: Deep representations and gradient
  descent can approximate any learning algorithm.
\newblock In \emph{International Conference on Learning Representations
  (ICLR)}, 2018.

\bibitem[Finn et~al.(2017{\natexlab{a}})Finn, Abbeel, and
  Levine]{finn2017model}
Chelsea Finn, Pieter Abbeel, and Sergey Levine.
\newblock Model-agnostic meta-learning for fast adaptation of deep networks.
\newblock In \emph{Proc. International Conference on Machine Learning (ICML)},
  pages 1126--1135, 2017{\natexlab{a}}.

\bibitem[Finn et~al.(2017{\natexlab{b}})Finn, Yu, Zhang, Abbeel, and
  Levine]{finn2017one}
Chelsea Finn, Tianhe Yu, Tianhao Zhang, Pieter Abbeel, and Sergey Levine.
\newblock One-shot visual imitation learning via meta-learning.
\newblock In \emph{Conference on Robot Learning (CoRL)}, pages 357--368,
  2017{\natexlab{b}}.

\bibitem[Finn et~al.(2018)Finn, Xu, and Levine]{finn2018probabilistic}
Chelsea Finn, Kelvin Xu, and Sergey Levine.
\newblock Probabilistic model-agnostic meta-learning.
\newblock In \emph{Advances in Neural Information Processing Systems
  (NeurIPS)}, pages 9516--9527, 2018.

\bibitem[Finn et~al.(2019)Finn, Rajeswaran, Kakade, and Levine]{finn2019online}
Chelsea Finn, Aravind Rajeswaran, Sham Kakade, and Sergey Levine.
\newblock Online meta-learning.
\newblock In \emph{International Conference on Machine Learning (ICML)}, pages
  1920--1930, 2019.

\bibitem[Foerster et~al.(2018)Foerster, Farquhar, Al-Shedivat, Rockt{\"a}schel,
  Xing, and Whiteson]{foerster2018dice}
Jakob Foerster, Gregory Farquhar, Maruan Al-Shedivat, Tim Rockt{\"a}schel, Eric
  Xing, and Shimon Whiteson.
\newblock {DiCE}: The infinitely differentiable monte carlo estimator.
\newblock In \emph{International Conference on Machine Learning (ICML)}, pages
  1529--1538, 2018.

\bibitem[Grant et~al.(2018)Grant, Finn, Levine, Darrell, and
  Griffiths]{grant2018recasting}
Erin Grant, Chelsea Finn, Sergey Levine, Trevor Darrell, and Thomas Griffiths.
\newblock Recasting gradient-based meta-learning as hierarchical bayes.
\newblock In \emph{International Conference on Learning Representations
  (ICLR)}, 2018.

\bibitem[Jerfel et~al.(2018)Jerfel, Grant, Griffiths, and
  Heller]{jerfel2018online}
Ghassen Jerfel, Erin Grant, Thomas~L Griffiths, and Katherine Heller.
\newblock Online gradient-based mixtures for transfer modulation in
  meta-learning.
\newblock \emph{arXiv preprint arXiv:1812.06080}, 2018.

\bibitem[Ji et~al.(2020)Ji, Lee, Liang, and Poor]{ji2020convergence}
Kaiyi Ji, Jason~D Lee, Yingbin Liang, and H~Vincent Poor.
\newblock Convergence of meta-learning with task-specific adaptation over
  partial parameters.
\newblock \emph{arXiv preprint arXiv:2006.09486}, 2020.

\bibitem[Kim et~al.(2020)Kim, Park, and Choi]{kim2020multi}
Jin-Hwa Kim, Junyoung Park, and Yongseok Choi.
\newblock Multi-step estimation for gradient-based meta-learning.
\newblock \emph{arXiv preprint arXiv:2006.04298}, 2020.

\bibitem[Koch et~al.(2015)Koch, Zemel, and Salakhutdinov]{koch2015siamese}
Gregory Koch, Richard Zemel, and Ruslan Salakhutdinov.
\newblock Siamese neural networks for one-shot image recognition.
\newblock In \emph{ICML Deep Learning Workshop}, volume~2, 2015.

\bibitem[Li et~al.(2017)Li, Zhou, Chen, and Li]{li2017meta}
Zhenguo Li, Fengwei Zhou, Fei Chen, and Hang Li.
\newblock Meta-{SGD}: Learning to learn quickly for few-shot learning.
\newblock \emph{arXiv preprint arXiv:1707.09835}, 2017.

\bibitem[Likhosherstov et~al.(2020)Likhosherstov, Song, Choromanski, Davis, and
  Weller]{likhosherstov2020ufo}
Valerii Likhosherstov, Xingyou Song, Krzysztof Choromanski, Jared Davis, and
  Adrian Weller.
\newblock {UFO-BLO}: Unbiased first-order bilevel optimization.
\newblock \emph{arXiv preprint arXiv:2006.03631}, 2020.

\bibitem[Liu et~al.(2019)Liu, Socher, and Xiong]{liu2019taming}
Hao Liu, Richard Socher, and Caiming Xiong.
\newblock Taming {MAML}: Efficient unbiased meta-reinforcement learning.
\newblock In \emph{International Conference on Machine Learning (ICML)}, pages
  4061--4071, 2019.

\bibitem[McLeod(1965)]{mcleod1965mean}
Robert~M McLeod.
\newblock Mean value theorems for vector valued functions.
\newblock \emph{Proceedings of the Edinburgh Mathematical Society}, 14\penalty0
  (3):\penalty0 197--209, 1965.

\bibitem[Mi et~al.(2019)Mi, Huang, Zhang, and Faltings]{mi2019meta}
Fei Mi, Minlie Huang, Jiyong Zhang, and Boi Faltings.
\newblock Meta-learning for low-resource natural language generation in
  task-oriented dialogue systems.
\newblock In \emph{Proceedings of the 28th International Joint Conference on
  Artificial Intelligence (IJCAI)}, pages 3151--3157, 2019.

\bibitem[Munkhdalai and Yu(2017)]{munkhdalai2017meta}
Tsendsuren Munkhdalai and Hong Yu.
\newblock Meta networks.
\newblock In \emph{International Conference on Machine Learning (ICML)}, pages
  2554--2563, 2017.

\bibitem[Naik and Mammone(1992)]{naik1992meta}
Devang~K Naik and Richard~J Mammone.
\newblock Meta-neural networks that learn by learning.
\newblock In \emph{IEEE International Joint Conference on Neural Networks
  (IJCNN)}, pages 437--442, 1992.

\bibitem[Nichol and Schulman(2018)]{nichol2018reptile}
Alex Nichol and John Schulman.
\newblock Reptile: a scalable metalearning algorithm.
\newblock \emph{arXiv preprint arXiv:1803.02999}, 2018.

\bibitem[Nichol et~al.(2018)Nichol, Achiam, and Schulman]{nichol2018first}
Alex Nichol, Joshua Achiam, and John Schulman.
\newblock On first-order meta-learning algorithms.
\newblock \emph{arXiv preprint arXiv:1803.02999}, 2018.

\bibitem[Raghu et~al.(2020)Raghu, Raghu, Bengio, and Vinyals]{raghu2020rapid}
Aniruddh Raghu, Maithra Raghu, Samy Bengio, and Oriol Vinyals.
\newblock Rapid learning or feature reuse? towards understanding the
  effectiveness of {MAML}.
\newblock In \emph{International Conference on Learning Representations
  (ICLR)}, 2020.

\bibitem[Rajeswaran et~al.(2019)Rajeswaran, Finn, Kakade, and
  Levine]{rajeswaran2019meta}
Aravind Rajeswaran, Chelsea Finn, Sham~M Kakade, and Sergey Levine.
\newblock Meta-learning with implicit gradients.
\newblock In \emph{Advances in Neural Information Processing Systems
  (NeurIPS)}, pages 113--124, 2019.

\bibitem[Ravi and Larochelle(2016)]{ravi2016optimization}
Sachin Ravi and Hugo Larochelle.
\newblock Optimization as a model for few-shot learning.
\newblock In \emph{International Conference on Learning Representations
  (ICLR)}, 2016.

\bibitem[Rothfuss et~al.(2019)Rothfuss, Lee, Clavera, Asfour, and
  Abbeel]{rothfuss2019promp}
Jonas Rothfuss, Dennis Lee, Ignasi Clavera, Tamim Asfour, and Pieter Abbeel.
\newblock {ProMP}: Proximal meta-policy search.
\newblock In \emph{International Conference on Learning Representations
  (ICLR)}, 2019.

\bibitem[Santoro et~al.(2016)Santoro, Bartunov, Botvinick, Wierstra, and
  Lillicrap]{santoro2016meta}
Adam Santoro, Sergey Bartunov, Matthew Botvinick, Daan Wierstra, and Timothy
  Lillicrap.
\newblock Meta-learning with memory-augmented neural networks.
\newblock In \emph{International Conference on Machine Learning (ICML)}, pages
  1842--1850, 2016.

\bibitem[Snell et~al.(2017)Snell, Swersky, and Zemel]{snell2017prototypical}
Jake Snell, Kevin Swersky, and Richard Zemel.
\newblock Prototypical networks for few-shot learning.
\newblock In \emph{Advances in Neural Information Processing Systems
  (NeurIPS)}, pages 4077--4087, 2017.

\bibitem[Song et~al.(2020)Song, Gao, Yang, Krzysztof, Pacchiano, and
  Tang]{song2020simple}
Xingyou Song, Wenbo Gao, Yuxiang Yang, Choromanski Krzysztof, Aldo Pacchiano,
  and Yunhao Tang.
\newblock {ES-MAML}: Simple hessian-free meta learning.
\newblock In \emph{International Conference on Learning Representations
  (ICLR)}, 2020.

\bibitem[Thrun and Pratt(2012)]{thrun2012learning}
Sebastian Thrun and Lorien Pratt.
\newblock \emph{Learning to learn}.
\newblock Springer Science \& Business Media, 2012.

\bibitem[Tripuraneni et~al.(2020)Tripuraneni, Jin, and
  Jordan]{tripuraneni2020provable}
Nilesh Tripuraneni, Chi Jin, and Michael~I Jordan.
\newblock Provable meta-learning of linear representations.
\newblock \emph{arXiv preprint arXiv:2002.11684}, 2020.

\bibitem[Vinyals et~al.(2016)Vinyals, Blundell, Lillicrap, Wierstra,
  et~al.]{vinyals2016matching}
Oriol Vinyals, Charles Blundell, Timothy Lillicrap, Daan Wierstra, et~al.
\newblock Matching networks for one shot learning.
\newblock In \emph{Advances in Neural Information Processing Systems
  (NeurIPS)}, pages 3630--3638, 2016.

\bibitem[Wang et~al.(2020{\natexlab{a}})Wang, Sun, and Li]{wang2020global}
Haoxiang Wang, Ruoyu Sun, and Bo~Li.
\newblock Global convergence and induced kernels of gradient-based
  meta-learning with neural nets.
\newblock \emph{arXiv preprint arXiv:2006.14606}, 2020{\natexlab{a}}.

\bibitem[Wang et~al.(2020{\natexlab{b}})Wang, Cai, Yang, and
  Wang]{wang2020global2}
Lingxiao Wang, Qi~Cai, Zhuoran Yang, and Zhaoran Wang.
\newblock On the global optimality of model-agnostic meta-learning.
\newblock In \emph{International conference on machine learning (ICML)},
  2020{\natexlab{b}}.

\bibitem[Williams(1992)]{williams1992simple}
Ronald~J Williams.
\newblock Simple statistical gradient-following algorithms for connectionist
  reinforcement learning.
\newblock \emph{Machine Learning}, 8\penalty0 (3-4):\penalty0 229--256, 1992.

\bibitem[Zhou et~al.(2019)Zhou, Yuan, Xu, Yan, and Feng]{zhou2019efficient}
Pan Zhou, Xiaotong Yuan, Huan Xu, Shuicheng Yan, and Jiashi Feng.
\newblock Efficient meta learning via minibatch proximal update.
\newblock In \emph{Advances in Neural Information Processing Systems
  (NeurIPS)}, pages 1532--1542, 2019.

\bibitem[Zintgraf et~al.(2018)Zintgraf, Shiarlis, Kurin, Hofmann, and
  Whiteson]{zintgraf2018caml}
Luisa~M Zintgraf, Kyriacos Shiarlis, Vitaly Kurin, Katja Hofmann, and Shimon
  Whiteson.
\newblock {CAML}: Fast context adaptation via meta-learning.
\newblock \emph{arXiv preprint arXiv:1810.03642}, 2018.

\end{thebibliography}

\end{document}